\documentclass[11pt]{article}
\usepackage{custom}

\usepackage{wrapfig}
\usepackage[table]{xcolor}
\usepackage{colortbl} 
\usepackage{tikz}
\usetikzlibrary{shapes, positioning, arrows.meta, backgrounds, fit, calc}

\usepackage{booktabs} 
\usepackage{subcaption}
\title{Optimal Decision-Making Based on Prediction Sets}
\author{
  Tao Wang \quad
  Edgar Dobriban \\
  University of Pennsylvania\thanks{Author e-mail addresses: \texttt{tawan@wharton.upenn.edu}, \texttt{dobriban@wharton.upenn.edu}}
}

\begin{document}

\maketitle


\begin{abstract}
Prediction sets can wrap around any ML model to cover unknown test outcomes with a guaranteed probability.
Yet, it remains unclear how to 
use them optimally for
downstream decision-making. 
Here,
we propose a decision-theoretic framework that seeks to minimize the expected loss (risk) against a worst-case distribution consistent with the prediction set's coverage guarantee. We first characterize the minimax optimal policy for a fixed prediction set, showing that it balances the worst-case loss inside the set with a penalty for potential losses outside the set. 
Building on this, we derive the optimal prediction set construction that minimizes the resulting robust risk subject to a coverage constraint. 
Finally, we 
introduce Risk-Optimal Conformal Prediction (ROCP), a practical algorithm that targets these risk-minimizing sets while maintaining finite-sample distribution-free marginal coverage.
Empirical evaluations on medical diagnosis and safety-critical decision-making tasks
demonstrate that ROCP reduces critical mistakes
compared to baselines, 
particularly when out-of-set errors are costly.
\end{abstract}


\section{Introduction}
When making predictions, we often need to take into account 
how they affect downstream decisions.
The classical statistical decision-theoretic prescription is clear: if the conditional law 
of the outcomes $Y$ given the features $x$,
$P(Y\mid X=x)$, were known, one would choose the Bayes action minimizing expected loss $\mathbb E[\ell(a,Y)\mid X=x]$ \citep{wald1945statistical,wald1949statistical,lehmann1998theory}.
In modern ML pipelines, however, the conditional law is unknown and predictive models are imperfect.
 This has motivated a great deal of work on optimal decision making under uncertainty, see e.g., \citet{keith2021survey,elmachtoub2022smart,zhao2021calibrating}, etc.
 However, many of these works still assume that one can obtain 
  ``correct" probability predictions in some weak sense, e.g., consistent or calibrated predictions. 

To handle the case that the data distribution is completely unknown,  
the area of distribution-free uncertainty quantification and conformal prediction has emerged \citep[e.g.,][etc]{Wilks1941,Wald1943,vovk1999machine,vovk2005algorithmic,lei2013distribution,angelopoulos2023conformal}. 
Given i.i.d. data (and even under the weaker assumption of exchangeability),
conformal methods output prediction sets $C(x)\subseteq\mathcal Y$ satisfying a finite-sample, distribution-free marginal guarantee $\Pr\{Y\in C(X)\}\ge 1-\alpha$ \citep{vovk2005algorithmic,angelopoulos2023conformal}. 
Yet coverage by itself does not specify how to act.
This motivates a central question at the prediction--action interface:
how should one make \emph{provably safe and effective decisions} when the only reliable information about $Y$ comes from a prediction set with a coverage
guarantee?

Recent work by \citet{kiyani2025decision} suggests
choosing a \emph{max--min} optimal action, 
minimizing the worst-case loss over $y\in C(x)$. 
They show that this rule 
is optimal for \emph{quantile-style} objectives, where the agent only cares about performance on a $1-\alpha$ fraction of outcomes \citep{kiyani2025decision}.
However, 
the more standard notion of expected loss
is sensitive to rare but catastrophic events: even if $Y\notin C(x)$ occurs with probability at most $\alpha$, the loss incurred on that event may be orders of magnitude larger than any in-set loss.
In such regimes, a purely in-set max--min rule can be brittle because it has no incentive to hedge against the $\alpha$ mass that is \emph{not} constrained by the set.

{\bf From prediction sets to minimax-optimal actions.}
This paper develops a decision-theoretic framework that makes this trade-off explicit.
Our starting point is a two-player game between the decision maker and nature.
For a fixed set $S\subseteq\mathcal Y$ and action $a\in\mathcal A$, nature may choose any distribution on $\mathcal Y$ that places at least $1-\alpha$ probability on $S$.
The resulting worst-case expected loss is the functio $L_S(a;\alpha)$ from \eqref{eq:wc-L}.
We show that this worst-case expectation admits a simple closed form (Lemma~\ref{lem:L-formula}):
$L_S(a;\alpha)=\ell^{\mathrm{in}}_S(a)+\alpha(\ell^{\mathrm{out}}_S(a)-\ell^{\mathrm{in}}_S(a))_+$;
where $\ell^{\mathrm{in}}$ and $\ell^{\mathrm{out}}$ are the maximum losses inside and outside of $S$, respectively.
This expression has a transparent interpretation.
The dominant term is the worst loss \emph{inside} $S$ (recovering the rule from \cite{kiyani2025decision}); but if the worst loss \emph{outside} $S$ is larger, the decision maker must pay an additional $\alpha$-weighted penalty.
Thus, unlike a pure max--min rule, the optimal action hedges against catastrophic out-of-set outcomes whenever they can materially affect expected risk.

Lifting this pointwise characterization to the original prediction-set pipeline, we derive the minimax-optimal policy $\pi^\star$ for any set-valued predictor $C$ (Theorem~\ref{thm:main}).
This yields decision rule that reduces to the in-set max--min rule when $\alpha=0$, but differs sharply in high-stakes regimes where out-of-set mistakes are costly.

{\bf Designing prediction sets for decision quality.}
The second half of the paper addresses the natural next question:
if prediction sets will be used to drive decisions, how should we choose them?
Decision-agnostic conformal sets are typically optimized for surrogate criteria such as size or top-$k$ mass \citep[see e.g,][etc]{Sadinle2019,romano2020classification,wang2025singleton}, but these objectives need not align with downstream loss.
Motivated by our minimax characterization, we formulate a population optimization problem that directly minimizes the robust risk induced by the optimal downstream decision rule, subject to a coverage constraint \eqref{eq:optimal set problem}; similarly to the one for the different quantile-based objective in \citet{kiyani2025decision}.
Using duality theory for integral functionals,  
we characterize the optimal coverage assignment
(Theorem~\ref{thm:pop-opt-general}).

{\bf A finite-sample algorithm with distribution-free coverage.}
To translate the oracle characterization into a practical procedure, we introduce \emph{Risk-Optimal Conformal Prediction (ROCP)}.
ROCP uses any black-box probabilistic model
to
construct estimates
of the population quantities in our oracle decision rule,
and then uses a held-out calibration set to ensure coverage 
via conformal prediction (Algorithm~\ref{alg:ROCP-EL}).
This guarantees finite-sample marginal coverage under exchangeability, while targeting the decision-theoretic optimum as the model improves.
Empirically, ROCP consistently reduces 
worst-case risks and critical mistake rates relative to risk-averse max--min baselines and best-response approaches, with the largest gains arising precisely in regimes where the $\alpha$ fraction of out-of-set mass can induce catastrophic losses.

\subsection{Related Work}
There is a great deal of related work. Due to space limitations, we discuss some of it in Appendix \ref{app:related-work}.

{\bf Decision making under set-valued uncertainty.}
A natural way to act given a set $C(x)$ is to choose an action that is robust to all $y\in C(x)$, leading to max--min decision rules.
This principle is central in robust optimization, where uncertainty sets replace probabilistic models \citep{chan2020inverse,chan2023inverse,chan2024conformal,patel2024conformal,johnstone2021conformal,yeh2024end}.
Our formulation is closely related but distinct: we do not assume $Y\in C(x)$ surely.
Instead, the uncertainty comes with a \emph{coverage constraint} that leaves an $\alpha$ fraction of probability mass unconstrained.
The resulting optimal rule is therefore not purely in-set robust: it is robust \emph{in expectation} over the worst-case distribution consistent with coverage, which produces the additional out-of-set penalty term in Lemma~\ref{lem:L-formula}.
This distinction is critical in applications where rare out-of-set errors have disproportionate cost.

{\bf Risk-averse objectives and our companion work.}
Prior work by \citet{kiyani2025decision} studies the prediction--action interface for risk-averse agents who optimize a quantile-style objective (value-at-risk) and proves that, under a marginal coverage constraint, the max--min policy is minimax-optimal for that criterion.
It also derives population-optimal prediction sets and a finite-sample algorithm (RAC) tailored to this quantile objective.
Our paper can be viewed as the expectation-risk counterpart: we replace the quantile objective by expected loss, which fundamentally changes the minimax structure.
In particular, expected loss forces the decision maker to account for the $\alpha$ mass outside the set whenever it can generate larger losses, leading to a different optimal policy and a different optimal set construction. 

Although the set-design problem in both this paper and \citet{kiyani2025decision} admit a one-dimensional dual parameter, the duality arguments are technically different. \citet{kiyani2025decision} reparametrize the coverage assignment via an auxiliary step function
reducing the problem to an infinite-dimensional linear program which by LP duality theory yields the threshold form solution. 
In contrast, our objective is an integral functional. 
To justify strong duality and the interchange of infimum and expectation, we work in the normal-integrand framework and invoke Fenchel-Rockafellar duality through a randomized-kernel relaxation, followed by 
derandomization to recover a deterministic optimizer. Empirically, this difference is most pronounced in high-stakes regimes with highly asymmetric losses, where small miscoverage probabilities can still dominate the expected risk.



\section{Optimal action selection from prediction sets}
\subsection{Setting}
We study how to minimize an expected loss, leveraging a prediction set. 
Let $\mathcal X$ and $\mathcal Y$ be the feature and outcome spaces, 
and let
$\mathcal A$ be the action space. 
For instance,
the actions could be the outcomes:
$\mathcal{A} = \mathcal{Y}$.

A prediction set is a map $C:\mathcal X\to 2^{\mathcal Y}$, assigning to each $x\in\mathcal X$ a set $C(x)\subseteq\mathcal Y$.
A \emph{policy} $\pi:2^{\mathcal Y}\times[0,1]\to\mathcal A$ takes as input a set $S\subseteq\mathcal Y$ together with a miscoverage parameter $\alpha'\in[0,1]$ (equivalently, coverage $t=1-\alpha'$), and outputs an action $\pi(S,\alpha')\in\mathcal A$.
When a single global miscoverage level $\alpha$ is fixed, we write $\pi(S):=\pi(S,\alpha)$ for brevity.
Let $\Pi$ denote the class of such policies.
The loss of the decision maker depends on both the chosen action $a$ and the label $y$,
and is captured by a fixed loss function $\ell:\mathcal A\times\mathcal Y\to
\mathbb R_+:=[0,\infty)$.

We assume that the decision maker has access to a prediction set $C(x)$ that is guaranteed to contain the true label with high probability. 
Specifically, in order to derive our oracle-optimal decision policies, 
we will consider
 an idealized 
prediction set 
that satisfies a conditional coverage property
$\mathbb{P}_{(X,Y)\sim P}\big(Y\in C(X)\mid X=x\big) \ge 1-\alpha,\ P_X\ \text{a.e.}\ x$ for some $\alpha\in(0,1)$.
 In practice, such conditional coverage properties are not generally possible for continuous feature spaces, but rather only possible for discrete feature spaces or approximately in continuous domains, see e.g., \citet{vovk2012conditional,lei2014distribution,barber2020limits,guan2023localized,gibbs2025conformal,joshi2025conformal}. 
 However, we emphasize that the idealization will only be considered for technical reasons in order to derive a clean form for the oracle optimal decision policies. 
 In practice, we will use those policies along with prediction sets satisfying  the standard marginal coverage conditions 
 $\mathbb{P}_{(X,Y)\sim P}\big(Y\in C(X)) \ge 1-\alpha$
\citep{vovk2005algorithmic,angelopoulos2023conformal}. 
We will see that,
despite motivating the oracle from a stronger condition, it will empirically lead to better decisions even under the standard realistic conditions. 
We typically consider $\alpha$ to be small, such as 0.05.

Given a prediction set, a canonical interpretation is that the outcomes in the prediction set are all plausible. 
However, how should one use these outcomes 
in a downstream decision-making task?
 Our goal is precisely to address this task,
 in the standard framework of statistical decision theory, which focuses on the risk, namely the expected loss \citep[see e.g.,][etc]{wald1945statistical,wald1949statistical,lehmann1998theory}. 
Our first goal is to characterize optimal decision policies subject to choosing actions from prediction sets with a fixed coverage level. 

Fix $\alpha\in[0,1)$, and define $\mathcal P_\alpha:=\mathcal P_\alpha(C)$ as the set of all the data distributions that are consistent with 
conditional coverage\footnote{From now on, $\text{a.e.}\ x$ always refers to $P_X\ \text{a.e.}\ x$ for brevity.}
\[
\scalebox{0.9}{$
\mathcal{P}_\alpha := \Big\{P\text{ on }\mathcal X\times\mathcal Y:
 P\big(Y\in C(X)\mid X=x\big) \ge 1-\alpha,\ \text{a.e.}\ x\Big\}.
$}
\]
 As mentioned above, 
to derive our oracle-optimal rule, 
we consider 
an idealized setting
where
the true distribution $P$ is unknown but it belongs to $\mathcal P_\alpha$. 
We aim to find the optimal policy minimizing the risk when choosing actions from the prediction set $C$, in the worst case  over all 
such distributions. 
This leads to the problem of solving, over all policies $\pi:2^\mathcal{Y}\to \mathcal{A}$:
\begin{equation}
\label{eq:main-game}
\inf_{\pi}\ \sup_{P\in\mathcal{P}_\alpha}\ 
\mathbb E_{(X,Y)\sim P}\Big[\ell\big(\pi(C(X)),Y\big)\Big].
\end{equation}
This can be viewed as a two-person game, played between the analyst, who wishes to choose the best policy $\pi$, and ``nature", which can set the data distribution $P$ to be adversarial. 
\citet{kiyani2025decision} study an analogous problem, minimizing an expected quantile of the loss 
under a marginal coverage constraint. 
In contrast,  we study minimizing the expectation of the loss; which is motivated by the standard formulation of statistical decision theory  \citep[see e.g.,][etc]{wald1945statistical,wald1949statistical,lehmann1998theory}. Our solution turns out to be different in intriguing ways.

\subsection{Optimal Policies}
For a set $S\subseteq\mathcal Y$ and $a\in\mathcal A$, define
$
\ell^{\mathrm{in}}_S(a):=\sup_{y\in S} \ell(a,y)$,
 to be the maximum\footnote{If $\mathcal Y$ is finite, the suprema in $\ell^{\mathrm{in}}_S$ and $\ell^{\mathrm{out}}_S$ are attained.
More generally, attainment holds if $\ell(a,\cdot)$ is upper semicontinuous and the optimization domain is compact
(e.g., if $S$ and $S^c$ are compact subsets of a compact $\mathcal Y$). Throughout we assume the suprema are finite,
and when we require attainment we state it explicitly.}
of the loss achieved by action $a$ over the set $S$.
 Similarly, define 
$\ell^{\mathrm{out}}_S(a):=\sup_{y\notin S} \ell(a,y)$
to be the maximum achieved outside of the set $S$, with the convention that $\ell^{\mathrm{out}}_{\mathcal Y}(a):=\ell^{\mathrm{in}}_{\mathcal Y}(a)$. 
For an action $a$ 
and a set $S\subseteq\mathcal Y$, define the worst-case expected loss---over probability distributions $Q$ for which the set $S$ achieves coverage---under a miscoverage level of $\alpha$ by
    \begin{equation}
    \label{eq:wc-L}
             L_S(a;\alpha) = \sup _{Q(\cdot ):\ Q( S ) \geq 1-\alpha} \mathbb{E}_{Y \sim Q}[\ell(a, Y)].
    \end{equation}
    We will only apply $L_S(a;\alpha)$ to sets $S$ for which the constraint set $\{Q:\ Q(S)\ge 1-\alpha\}$ is nonempty (in particular, $S\neq\varnothing$ whenever $\alpha<1$).\footnote{When $S = \emptyset$ and $\alpha<1$, there is no probability distribution satisfying the constraint. Therefore, we can define $L_S(a;\alpha)$ as $\infty$.}

This notion
will turn out to be an important intermediate quantity in our analysis, because it characterizes the worst-case expected loss that can be induced by a probability distribution subject to coverage; for a fixed predicted set $S\subset \mathcal{Y}$.
 We will use this repeatedly in our development. 
 Therefore, it will be important to find a simpler expression for it.
 Fortunately, it turns out that this is possible, as shown by the following result. 
\begin{lemma}[Worst-case expected loss under miscoverage level of $\alpha$]
\label{lem:L-formula}
For any $S\subseteq\mathcal Y$ and any $a\in\mathcal A$, 
\begin{equation}
    \label{eq:worst-case loss form}L_S(a;\alpha)=\ell_S^{\mathrm{in }}(a)+\alpha\left(\ell_S^{\mathrm{out }}(a)-\ell_S^{\mathrm{in }}(a)\right)_{+}    
\end{equation}
where $(t)_+=\max\{t,0\}$.
\end{lemma}
 The formula in \eqref{eq:worst-case loss form} has an insightful interpretation: the worst-case loss first looks at the worst-case outcome inside the prediction set $S$ (i.e., $\ell_S^{\mathrm{in }}(a)$), which represents the worst-case loss over the more likely set with probability at least $1-\alpha$.
 This is then compared with the worst-case outcome outside of the prediction set $S$---through $\ell_S^{\mathrm{out }}(a)$---and if the loss outside is larger, then a penalty of $\left(\ell_S^{\mathrm{out }}(a)-\ell_S^{\mathrm{in }}(a)\right)_{+} $---multiplied by the typically small $\alpha$ is added. 
 In other words, the worst-case loss is mainly determined by the losses within the prediction set, but it is also penalized (by a small amount) by the losses outside of the prediction set.  
 In contrast, the solution to the analogous problem for the quantile from \citet{kiyani2025decision} amounts to only the first component, namely $\ell_S^{\mathrm{in }}(a)$.
The proof (with all proofs) is deferred to the Appendix \ref{app:proofs}.

Equipped with this result, we can now characterize\footnote{We tacitly restrict to measurable set-valued predictors $C$ and policies $\pi$ for which
$\ell(\pi(C(X)),Y)$ is measurable; see Appendix~\ref{app:proofs} for formal conventions.} 
the optimal policy and the worst-case distribution for our original problem from \eqref{eq:main-game}:

\begin{theorem}[Optimal policy and risk]
\label{thm:main}
If for every $x\in\mathcal{X}$, the function $a\mapsto L_{C(x)}(a;\alpha)$ attains its minimum,\footnote{If $\mathcal A$ is compact and $a\mapsto \ell(a,y)$ is lower semicontinuous for each $y$, then the function $a\mapsto L_S(a;\alpha)$ attains its minimum.} 
then 
optimal policies $\pi^\star$ to the problem \eqref{eq:main-game} have the form\footnote{Existence of measurable selection is discussed in Appendix~\ref{app:proofs}; see Remark~\ref{rem:measurable-pi-star} 
after the proof of Theorem~\ref{thm:main}.}
\begin{equation}\label{pi-star}
    \pi^\star(C(x)) \in \arg\min_{a\in\mathcal{A}} L_{C(x)}(a;\alpha),\ x\in \mathcal{X},
\end{equation}
and the minimax risk is
\begin{equation}
\label{eq:value}
\inf_{\pi}\ \sup_{P\in\mathcal{P}_\alpha}\ \mathbb E\,\ell\big(\pi(C(X)),Y\big)
\;=\; \sup_{x\in \mathcal{X}}\ \min_{a\in\mathcal A}\ L_{C(x)}(a;\alpha).
\end{equation}
Moreover, if in addition the suprema in $\ell^{\mathrm{in}}_S$ and $\ell^{\mathrm{out}}_S$ are attained for all $a$, and the outer supremum over $x$ in \eqref{eq:value} is attained,
then a worst-case $P^\star\in\mathcal P_\alpha$ can be chosen 
as follows.
Let $P_X=\delta_{x^\star}$ for some
$x^\star\in\arg\max_{x\in\mathcal X}\min_{a\in\mathcal A}L_{C(x)}(a;\alpha)$ and, writing $S^\star=C(x^\star)$,
\[
Y|X=x^\star \sim
\begin{cases}
(1-\alpha)\delta_{y_{\mathrm{i}}}+\alpha\delta_{y_{\mathrm{o}}}, \text{if } S^\star\ne\mathcal Y\ \text{\&}\ \ell^{\mathrm{out}}_{\star}>\ell^{\mathrm{in}}_{\star},\\
\delta_{y_{\mathrm{i}}}, \text{otherwise,}
\end{cases}
\]
where $\ell^{\mathrm{in}}_{\star}=\ell^{\mathrm{in}}_{S^\star}(\pi^\star(S^\star))$ and $\ell^{\mathrm{out}}_{\star}=\ell^{\mathrm{out}}_{S^\star}(\pi^\star(S^\star))$,
and $y_{\mathrm{i}}\in\arg\max_{y\in S^\star}\ell(\pi^\star(S^\star),y)$ while (in the first case) $y_{\mathrm{o}}\in\arg\max_{y\notin S^\star}\ell(\pi^\star(S^\star),y)$.

\end{theorem}
Theorem \ref{thm:main} states that the adversary will always concentrate all the feature mass at a single point $x^\star$, and the resulting risk reduces to $\inf_{a\in\mathcal{A}}L_{C(x^\star)}(a;\alpha)$. Hence, when the decision maker wants to make the decision based on prediction sets $C(x)$, $x\in \mathcal{X}$ that contain the actual label with high probability $1-\alpha$, the minimax optimal policy and the corresponding per-$x$ robust risk are 
\[
a^\star(x):=\pi^\star(C(x)),\quad R(C(x),\alpha):=\min_{a\in\mathcal{A}} L_{C(x)}(a;\alpha).
\]
Now we turn to the question of designing prediction sets that are suitable for decision-making. 
First,
motivated by the above characterization, 
we consider a hypothetical setting where the true distribution $P$ was known. 
We will later show how to apply this idea when $P$ is unknown.
As we already mentioned, the oracle optimal 
decision-theoretic characterization is phrased in terms of \emph{conditional} miscoverage, which is generally impossible \citep{vovk2012conditional, foygel2021limits}.
In practice, we will apply the method to a prediction set with a standard \emph{marginal} coverage guarantee $\mathbb{P}_P\left(Y\in C(X)\right) \ge 1-\alpha$; and we will argue empirically that the performance improves.

In the oracle set-design problem below (where $P$ is known), we allow the decision rule to depend on a
\emph{coverage assignment} $t:\mathcal X\to[0,1]$, interpreted as the conditional coverage level that is
certified at covariate value $x$. Concretely, we require that
$\Pr\{Y\in C(X)\mid X=x\}\ge t(x)$ for $P_X$-a.e.\ $x$.
By the tower property, this implies the marginal coverage bound
$\Pr\{Y\in C(X)\}\ge \mathbb E[t(X)]$.
Given such a certified level $t(x)$, the corresponding local miscoverage budget is $1-t(x)$ and the
robust risk at $x$ is $R(C(x),1-t(x))=\min_{a\in\mathcal A} L_{C(x)}(a;1-t(x))$.
We therefore consider the oracle set-design problem
\begin{equation}
\label{eq:optimal set problem}
 \begin{aligned}
    &\min_{C(\cdot),\,t(\cdot)} \quad \mathbb{E}\!\left[R\big(C(X),1-t(X)\big)\right]\\
    & \operatorname{s.t.} \ \Pr\{Y\in C(X)\mid X=x\}\ge t(x)\ \text{for $P_X$-a.e.\ }x,\\
    & \qquad \mathbb E[t(X)] \ge 1-\alpha.
\end{aligned}
\end{equation}
Next, we will explain how to solve this problem. 

\section{Oracle Optimality}
\label{ops}
Here,
we study the oracle optimal sets from \eqref{eq:optimal set problem}.
We start by following the approach from \cite{kiyani2025decision},
and then make the necessary changes to our setting.

Starting from this section, we drop the subscript $P$ on $\mathbb P_P$ and $\mathbb E_P$ for simplicity.
Our analysis begins with a \emph{pointwise} problem: fix a feature value $x\in\mathcal X$ and a target
certified conditional coverage level $t\in(0,1]$.
As in \cite{kiyani2025decision},
we will design a set $C(x)\subseteq\mathcal Y$
subject to the constraint $\Pr\{Y\in C(x)\mid X=x\}\ge t$, and evaluate decisions under the corresponding
miscoverage budget $1-t$ via the robust risk $R(C(x),1-t)=\min_{a\in\mathcal A}L_{C(x)}(a;1-t)$.

For $a\in\mathcal A$, define the maximal loss that action $a$ can incur as
$M(a):=\sup_{y\in\mathcal Y}\ell(a,y)<\infty$,
and define the conditional $t$-quantile of the loss at $(x,a)$ as
\[
Q_t^x(a) := \inf\big\{\theta\in\mathbb R:\ \mathbb{P}(\ell(a,Y)\le \theta\,\mid\,X=x)\ge t\big\}.
\]
 Also, consider the 
loss sublevel set for some $\theta$, namely
$S_\theta(a) := \{y\in\mathcal Y:\ \ell(a,y)\le \theta\}$. Intuitively, $S_{Q_t^x(a)}(a)$ 
are 
the lowest-loss 
labels under action $a$,
with $x$-conditional coverage of at least $t$. 
This is the natural candidate feasible set with small worst-case
in-set
loss $\sup_{y\in C(x)}\ell(a,y)$.

Moreover, under the constraint $\mathbb P(Y\in C(x)\mid X=x)\ge t$, an adversary in the definition of $L_{C(x)}(a;1-t)$ may place probability $t$ on a worst in-set label (with loss $Q_t^x(a)$ when $C(x)=S_{Q_t^x(a)}(a)$) and the remaining probability $1-t$ on a label achieving the maximal loss $M(a)$, leading to the pointwise objective $t\,Q_t^x(a)\ +\ (1-t)\,M(a)$. 
Then the corresponding pointwise optimal action is
\begin{equation}
    \label{eq:pointwise-optimal action}
    a(x, t) \in \arg \min _{a \in \mathcal{A}}\left\{t Q_t^x(a)+(1-t) M(a)\right\}.
\end{equation}
This is analogous to the optimal action $a$ from \cite{kiyani2025decision} for the quantile setting. 
The associated threshold, below which the loss values are included, is $\theta(x, t):=Q_t^x(a(x, t))$;
and the resulting optimal set $C(x,t)$ is
\begin{equation}
    \label{eq:pointwise-optimal set}
    S_{\theta(x,t)}\big(a(x,t)\big) = \big\{y\in\mathcal Y: \ell\big(a(x,t),y\big)\le \theta(x,t)\big\}.
\end{equation}
The following proposition summarizes this formally.
It is an analogue of Proposition 3.1 in \cite{kiyani2025decision}, with the difference that our formulation for expected risk necessitates explicitly accounting for the $(1-t)$ probability mass falling outside the set, yielding the additional $(1-t)M(a)$ term absent from their formulation.
\begin{proposition}
\label{prop:EL-pointwise}
If the minimum in \eqref{eq:pointwise-optimal action} exists,\footnote{This holds, e.g., if $\mathcal A$ is compact and $a\mapsto Q_t^x(a)$ and $a\mapsto M(a)$ are lower semicontinuous.} 
then, subject to $\operatorname{Pr}(Y \in C \mid X=x) \geq t$, the set $C(x,t)$ from \eqref{eq:pointwise-optimal set} achieves the smallest risk $R(C,1-t)$ with
\[
R(C(x,t),1-t)\ 
=\ t\,\theta(x,t)\ +\ (1-t)\,M\big(a(x,t)\big).
\]
Further, if the suprema defining $\ell^{\mathrm{in}}_{C(x,t)}(a(x,t))$ and $M(a(x,t))$ are attained,
then the worst--case conditional law achieving the inner supremum in $L_{C(x,t)}(a(x,t);1-t)$ can be taken to place probability $t$ on a point $y_{\mathrm{i}}\in\arg\max_{y\in C(x,t)}\ell(a(x,t),y)$ and probability $1-t$ on a point $y_{\mathrm{o}}\in\arg\max_{y\in\mathcal Y}\ell(a(x,t),y)$.

\end{proposition}
\begin{remark}[Edge cases]
At $t=0$ the constraint is vacuous, i.e., every measurable $C\subseteq\mathcal Y$ is feasible, and the value reduces to $\min_{a\in\mathcal A} M(a)$, independent of $C$. For later use, when $t=0$ we fix any minimizer $a(x,0)\in\arg\min_{a\in\mathcal A} M(a)$ and set
$\theta(x,0):=M(a(x,0))$ and $C(x,0):=\mathcal Y$.
\end{remark}

Proposition \ref{prop:EL-pointwise} allows us to reparametrize the problem \eqref{eq:optimal set problem} in terms of the pointwise coverage $t:x\mapsto t(x)$, as in \cite{kiyani2025decision}.
Formally, the problem \eqref{eq:optimal set problem}  has the following equivalent reparametrization:
\begin{equation}
\label{eq:pop-problem}
\operatorname{VAL}(\alpha) := \inf_{\substack{t:\mathcal X\to[0,1]\ \text{measurable}\\ \mathbb E[t(X)]\ge 1-\alpha}}
\ \mathbb E\!\left[V_X\big(t(X)\big)\right].
\end{equation}
where we define
\begin{equation}
    \label{eq:v_x(t)}
    V_x(t) := \min_{a\in\mathcal A}\ \Big\{ t\,Q_t^x(a)\ +\ (1-t)\,M(a)\,\Big\},\quad t\in(0,1],
\end{equation}
and set $V_x(0):=\min_{a\in\mathcal A} M(a)$ as in Remark~\ref{prop:EL-pointwise}. Letting $t^\star$
be the optimum, 
the optimal actions are
$a^\star(x)=a\left(x, t^\star(x)\right)$,
and the optimal prediction set is:
\begin{equation}
\label{eq:optimal set}
C^\star(x) = C\big(x,t^\star(x)\big)
=\ \Big\{y\in\mathcal Y:\ \ell\big(a(x,t^\star(x)),y\big)\le \theta\big(x,t^\star(x)\big)\Big\}.
\end{equation}
To solve this problem, 
we adopt a duality-based approach in the spirit of \citet{kiyani2025decision}. However, their proof strategy does not transfer directly to our setting. They  
reduce the problem to an infinite-dimensional linear program and leverage LP duality theory, which does not work for our problem.
Instead, we introduce a different technical approach: we work in the normal-integrand framework and invoke Fenchel-Rockafellar duality through a randomized-kernel relaxation, followed by 
derandomization to recover a deterministic optimizer. 

For $\beta \ge0$, we define the dual function 
\[
\phi(\beta) := \beta(1-\alpha)\ +\ \mathbb E\!\left[\inf_{u\in[0,1]}\bigl\{V_X(u)-\beta u\bigr\}\right].
\]
and define the (set-valued) argmin correspondence
    \(
    \Gamma_\beta(x):=\arg\min_{u\in[0,1]}\bigl\{V_x(u)-\beta u\bigr\},
    \)
    and its extremal selectors
    \[
    g^+(x,\beta):=\max\Gamma_\beta(x),\qquad g^-(x,\beta):=\min\Gamma_\beta(x).
    \]
The following theorem characterizes the optimal $t^*$ and hence the optimal action and prediction set.

\begin{theorem}
\label{thm:pop-opt-general}
    Assume $P_X$ is non-atomic, that $(x,t)\mapsto V_x(t)$ in \eqref{eq:v_x(t)} is a normal integrand\footnote{See Appendix~\ref{app:proofs} for the definition of normal integrands and its equivalent characterization.}, and that, for each $x$, the minimum in the definition of $V_x(t)$ exists for every $t\in[0,1]$.
    Then there exists $\beta^*\ge 0$ and a measurable $t^\star:\mathcal X\to[0,1]$ such that
    \[
t^\star(x)\in\Gamma_{\beta^\star}(x)\quad \text{$P_X$-a.e.\ $x$},
    \]
    and $t^*$ solves the population problem \eqref{eq:pop-problem}. Substituting this optimal coverage assignment $t^\star(x)$ into \eqref{eq:optimal set} yields the optimal prediction sets, with actions $a^\star(x)=a\big(x,t^\star(x)\big)$ as in \eqref{eq:pointwise-optimal action}.  Moreover, one can always choose $t^*$ of the form
    \begin{equation}
        \label{eq:general t*}
            t^*(x)=
    \begin{cases} 
    g^-(x,\beta^\star)+\bigl(g^+(x,\beta^\star)-g^-(x,\beta^\star)\bigr)\mathbf 1_A(x), & \beta^*>0,\\
g^+(x,0), & \beta^*=0.
\end{cases}
    \end{equation}
    for some measurable $A\subseteq\mathcal{X}$. Any maximizer $\beta^{\star}>0$ of  the dual function $\phi(\beta)$ satisfies the interval condition
    \begin{equation}
        \label{eq:interval condition}
        \mathbb{E}\left[g^{-}\left(X, \beta^{\star}\right)\right] \leq 1-\alpha \leq \mathbb{E}\left[g^{+}\left(X, \beta^{\star}\right)\right].
    \end{equation}
    If $\beta^{\star}=0$, only the right inequality is required, i.e.\ $\mathbb{E}[g^{+}(X,0)]\ge 1-\alpha$.
\end{theorem}

\section{Risk--optimal conformal prediction}
\label{sec:ROCP-EL}
Theorem~\ref{thm:pop-opt-general} characterizes an oracle optimal prediction set, 
assuming the true distribution is known. 
This section constructs a conformal prediction set that emulates this oracle based on data.
 While the approximation guarantee to the optimum is, in general, challenging to establish,
 the conformal prediction set 
provides a finite-sample guarantee--namely, we have marginal coverage $1-\alpha$ under exchangeability.

\subsection{Constructing estimators}
\label{subsec:oracle-plug-in}

Given calibration data $\{(X_i,Y_i)\}_{i=1}^n$ 
and a 
probabilistic predictor
$f: \mathcal{X}\mapsto \Delta_{\mathcal{Y}}$, 
where we consider $f_x(y)$ to estimate the true conditional distribution $P(y|x)$. 
We adopt the approach of estimating the population-level quantities with plug-in estimators and then conformalizing the final result to guarantee coverage. 
Such approaches have been widely used in conformal prediction, 
see e.g., \cite{Sadinle2019,romano2020classification,kiyani2025decision, wang2025singleton}, etc.

Following the oracle construction from Section \ref{ops},
as in \cite{Sadinle2019,romano2020classification,kiyani2025decision, wang2025singleton},
we first estimate $Q_t^x(a)$ using a standard plug-in principle leveraging $f_x$:
$
\hat{Q}_t^x(a) := \inf\big\{\theta\in\mathbb R: \mathbb{P}_{Y\sim f_x}(\ell(a,Y)\le \theta)\ge t\big\},
$
then estimate
$
\widehat a(x,t)\in\arg\min_{a\in\mathcal A} \Big\{t\widehat Q_t^x(a)+(1-t) M(a)\Big\},$
and set $\widehat\theta(x,t):=\widehat Q_t^x\big(\widehat a(x,t)\big)$. Finally, define the plug-in estimate of $V_x$:
$$\widehat V_x(t):=\min_{a\in\mathcal A}\Big\{t\,\widehat Q_t^x(a)+(1-t)\,M(a)\Big\}.$$
and the empirical dual selector:
$\widehat g(x,\beta)\in\arg\min_{t\in[0,1]}\ \big\{\widehat V_x(t)-\beta t\big\}.$
Following 
Theorem \ref{thm:pop-opt-general} and as in \cite{kiyani2025decision}, 
we define the \emph{$\beta$--parametrized}  quantities:
\(
\widehat\theta(x,\beta) := \widehat\theta\!\big(x,\widehat g(x,\beta)\big)\),
\(\widehat a(x,\beta) := \widehat a\!\big(x,\widehat g(x,\beta)\big)\).
Following \eqref{eq:optimal set}, the plug-in estimate of the optimal set is  $\widehat C(x;\beta):=\widehat C_0\big(x;\widehat g(x,\beta)\big)$, where
$\widehat C_0(x;t):=\Big\{y\in\mathcal Y:\ \ell\big(\widehat a(x,t),y\big)\le \widehat\theta(x,t)\Big\}$.

\subsection{Risk-optimal conformal prediction}\label{subsec:ROCP-EL}
We present a method (Algorithm \ref{alg:ROCP-EL}) that we call
risk-optimal conformal prediction (ROCP), 
in the spirit of
 group conditional conformal prediction \citep{gibbs2025conformal} and 
RAC \citep{kiyani2025decision}.
The algorithm only uses a calibration set and the  maps $(\widehat\theta,\widehat a,\widehat g)$; 
it makes no assumptions about how $f$ was trained.
\begin{algorithm}[t!]
\caption{Risk--optimal conformal prediction (ROCP)}\label{alg:ROCP-EL}
\begin{algorithmic}[1]
\Require Calibration samples $\{(X_i,Y_i)\}_{i=1}^n$, test covariate $X_{\rm test}$.
\For{each candidate label $y \in \mathcal{Y}$}
\Statex {\small Solve $\displaystyle
\widehat\beta_y=\underset{\beta\ge 0}{\arg\min}\ \beta \quad \text{subject to:}\quad
\frac{1}{n+1}\Big\{\sum_{i=1}^n \mathbf{1}\{Y_i\in \widehat C(X_i;\beta)\}+\mathbf{1}\{y\in \widehat C(X_{\rm test};\beta)\}\Big\}\ge 1-\alpha.$}

\EndFor
\Ensure Prediction set $C_{\rm ROCP}(X_{\rm test}) := \{y\in\mathcal Y : y \in \widehat C(X_{\rm test}; \widehat\beta_y)\,\}$ and robust action $\widehat a_{\rm ROCP}(X_{\rm test})\in\arg\min_{a\in\mathcal A} L_{\,C_{\rm ROCP}(X_{\rm test})}(a;\alpha)$ .
\end{algorithmic}

\end{algorithm}

As in \cite{gibbs2025conformal,kiyani2025decision}, it follows that 
if the test datapoint
$(X_{\rm test},Y_{\rm test})$ and the 
calibration data $\{(X_i,Y_i)\}_{i=1}^n$ are exchangeable, the ROCP set satisfies the standard marginal coverage guarantee \citep{vovk1999machine,vovk2005algorithmic}
\(
\Pr\!\big\{Y_{\rm test}\in C_{\rm ROCP}(X_{\rm test})\,\big\}\ \ge\ 1-\alpha,
\)

{\bf Remark.}
If the  estimator $\widehat Q$ is consistent and the argmin in $a$ and in $t$ is stable, then 
we expect 
$(\widehat\theta,\widehat a,\widehat g)$ to approximate $(\theta,a,g)$ well.
In this case, we expect ROCP to approximate the optimal risk performance characterized by Theorem~\ref{thm:pop-opt-general}, while maintaining finite–sample coverage.

\section{Experiments}
Given any set $S\subseteq\mathcal Y$, define the robust decision rule
\[
a_{\rm ROCP}(S)\in\arg\min_{a\in\mathcal A} L_S(a;\alpha).
\]
In this section, we report experiments comparing our \texttt{ROCP} method with several baselines: 

{\bf Risk-Averse Calibration.} \citet{kiyani2025decision}
 study a setting with a utility $u$, which is equivalent to the negative loss, $u=-\ell$.
They propose the max--min decision rule
$$a_{\mathrm{RA}}(C(x)) \in \arg\max_{a\in\mathcal A}\ \min_{y\in C(x)} u(a,y),$$
which maps a prediction set $C(x)$ to an action by maximizing the worst-case utility over labels in the set. 
In our language, this corresponds to the rule $\arg\min_{a}\ell_{C(x)}^{\mathrm{in }}(a)$, which corresponds to the solution $\pi^\star$  from \eqref{pi-star} with $\alpha=0$.
They show that this rule is minimax-optimal for maximizing an expected quantile of the utility, and they derive prediction sets tailored to the max--min rule via their Algorithm~1. In our experiments, we evaluate both their full procedure \texttt{RAC} (Algorithm~1 paired with the max--min decision rule) and the max--min decision rule applied to prediction sets produced by other methods.

{\bf Calibration + Best-Response.}
We first calibrate the predictive model using \emph{decision calibration} on the calibration set \citep{zhao2021calibrating}. We then take the \emph{best-response} action under the (approximately calibrated) predictive distribution, $\operatorname{best-response}(x)  \in\ \arg\min_{a\in\mathcal A}\ \mathbb E_{Y\sim f_x}\big[\ell(a,Y)\big]$.
This baseline treats the model’s predictive probabilities as reliable and commits to a single action without accounting for set-valued uncertainty. We include it to illustrate the consequences of fully trusting the model: while it can achieve strong average utility when the predictive distribution is accurate, it may also incur frequent critical mistakes compared to our method.

{\bf Conformal Prediction.}
As decision-agnostic baselines, we construct 
prediction sets 
with marginal $(1-\alpha)$-coverage
via split conformal prediction using three scoring rules: Least Ambiguous Sets (\texttt{LAS}) \citep{Sadinle2019}, Adaptive Prediction Sets (\texttt{APS}) \citep{romano2020classification}, and \texttt{SOCOP} \citep{wang2025singleton}. 
Given some $x$ and for each conformal set $C(x)$, we then instantiate downstream decisions using both our action rule $a_{\text{ROCP}}(C(x))$ and the risk-averse max--min rule $a_{\text{RA}}(C(x))$ for comparison.

We evaluate the following metrics on the test set $\{(x_i,y_i)\}_{i=1}^{n_{\rm test}}$, where each method outputs a prediction set $C(x_i)$ and a corresponding action $a_i$. 
\begin{itemize}
    \item[(a)] {\bf Average Realized Worst-Case Risk}: the mean worst-case robust risk
$\frac{1}{n_{\rm test}}\sum_{i=1}^{n_{\rm test}} L_{C(x_i)}(a_i;\alpha),$
where $L_S(a;\alpha)$ is defined in Lemma~\ref{lem:L-formula} and  $a_i$ is the chosen action.
\item[(b)] {\bf Average Realized Loss}:
the test-time mean realized loss
$
\frac{1}{n_{\rm test}}\sum_{i=1}^{n_{\rm test}} \ell(a_i,y_i),$
where $y_i$ is the true label of $x_i$ and $a_i$ is the chosen action.
\item[(c)] {\bf Average Miscoverage}:
the empirical miscoverage rate
$\frac{1}{n_{\rm test}}\sum_{i=1}^{n_{\rm test}} \mathbf{1}\{y_i\notin C(x_i)\}.$
\item[(d)] {\bf Critical Mistake Rates}:
For each critical label $y_c$, we report the fraction of test
data with true label $y_c$ for which the chosen action attains the worst possible loss for that label:
$
\sum_{i:\ y_i=y_c}\mathbf{1}\Big\{a_i \in \arg\max_{a\in\mathcal A}\ell(a,y_c)\Big\}/{|\{i:\ y_i=y_c\}|}.$
\end{itemize}

\subsection{Medical Diagnosis}

\begin{figure}
\centering
\begin{subfigure}[t]{0.5\linewidth}
  \centering
  \includegraphics[width=\linewidth]{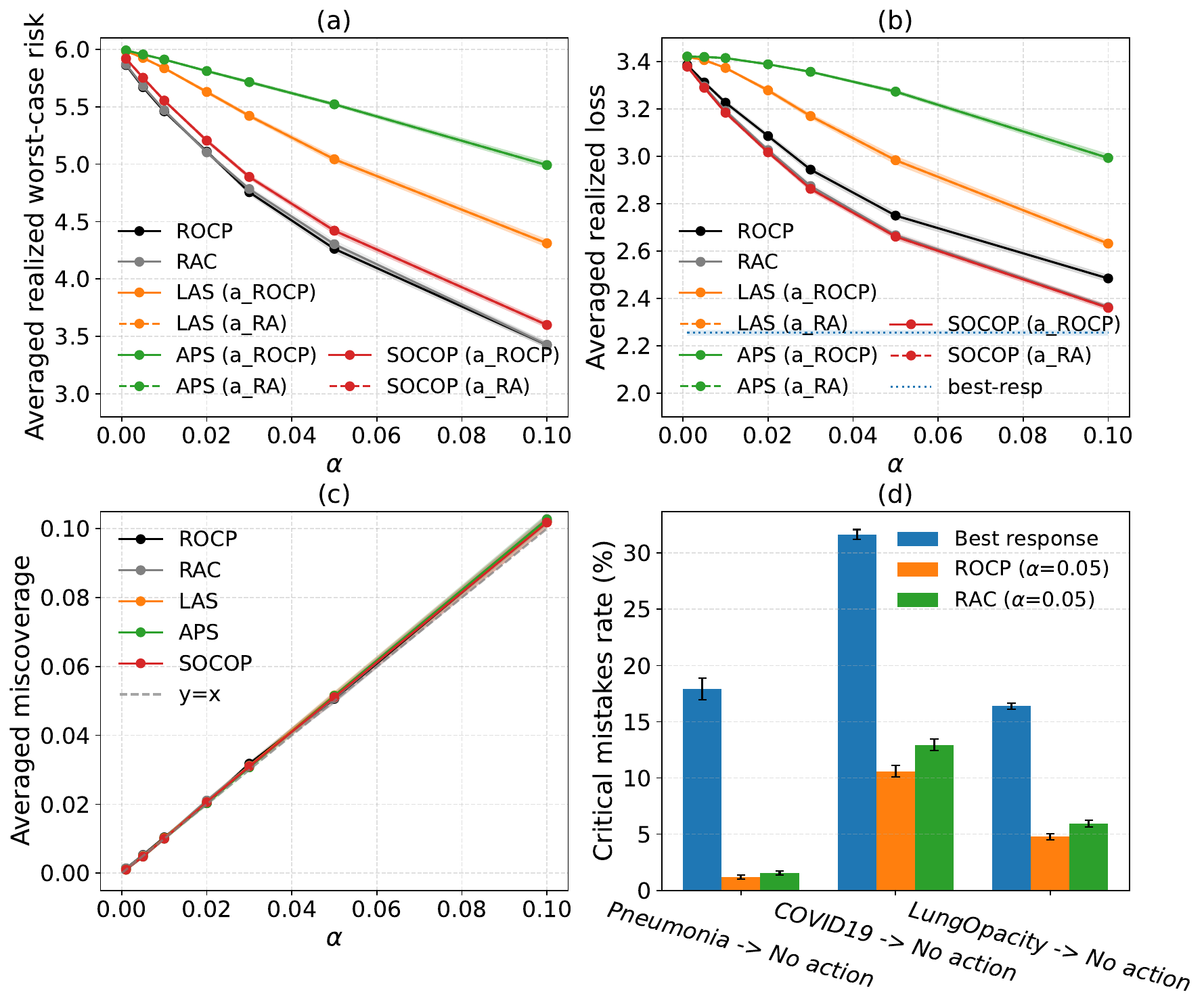}
  \caption{Baseline loss matrix $\Lambda_0$}
  \label{fig:covid-lambda0}
\end{subfigure}\hspace{0em}%
\begin{subfigure}[t]{0.5\linewidth}
  \centering
  \includegraphics[width=\linewidth]{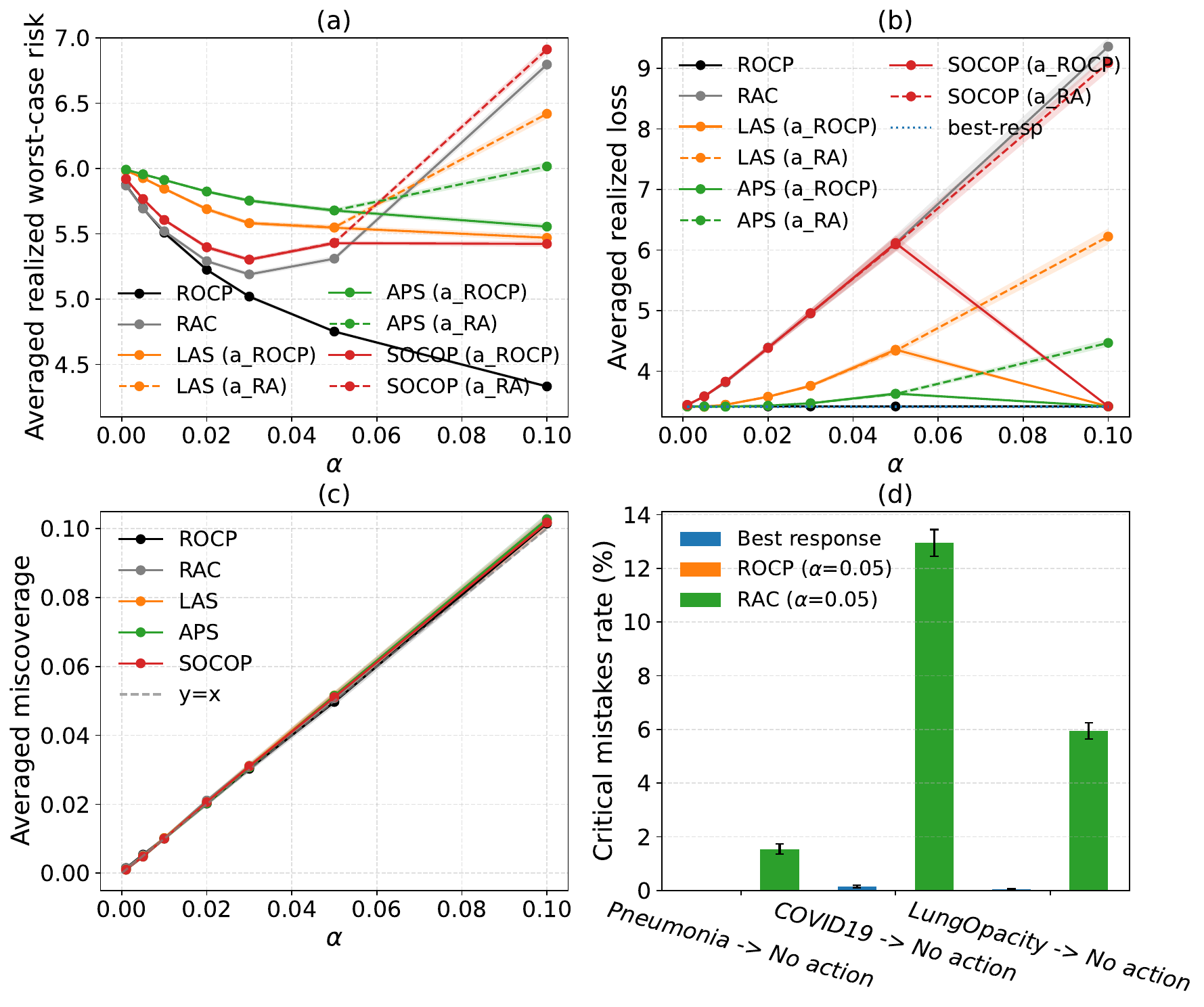}
  \caption{Loss matrix with $10\times$ more severe-mismatch penalties $\Lambda_1$}
   \label{fig:covid-lambda1}
\end{subfigure}
\caption{Medical diagnosis experiments. Results under two treatment-loss specifications: the baseline loss matrix from \citet{kiyani2025decision} (left) and our safety-critical variant. 
Each panel reports, as a function of miscoverage level $\alpha$: (a) average realized worst-case risk certificate;  (b) average realized loss; (c) empirical miscoverage; (d) critical mistake rate for critical labels, defined as the fraction of test points with true label $y_c$ for which the chosen action attains $\arg\max_{a\in\mathcal A}\ell(a,y_c)$. All results are averaged over 20 random train/calibration/test splits; error bars show $\pm 1$ standard error.}
\label{fig:covid-eval}

\end{figure}

First, in order to ensure a sufficiently detailed comparison with the RAC method of \citet{kiyani2025decision}, we start by comparing our method on a medical diagnosis example replicated from their paper. 
The data come from the COVID-19 Radiography Database \citep{chowdhury2020can,rahman2021exploring}, which contains chest X-ray images labeled into four categories: Normal, Pneumonia, COVID-19, and Lung Opacity. We randomly partition the dataset into training (70\%), calibration (10\%), and test (20\%) splits. For the predictive model, we use an \texttt{Inception-v3} architecture \citep{szegedy2015going,  szegedy2016rethinking} initialized with ImageNet-pretrained weights and fine-tune it on the training split.

{\bf Loss matrix designs.}
We model the downstream treatment objective using a loss matrix\footnote{RAC is formulated in terms of a utility matrix. We convert losses to utilities via $u(a,y)=M-\ell(a,y)$, where $M=\max_{a\in\mathcal A,\,y\in\mathcal Y}\ell(a,y)$.} $\Lambda\in\mathbb R_+^{|\mathcal Y|\times|\mathcal A|}$, 
where $\ell(a,y)=\Lambda_{y,a}$.
Here the labels are $y\in\{$ 
Normal,  Pneumonia,  COVID-19,  Lung Opacity$\}$ and the available actions are $a\in\{$
No Action,  Antibiotics,  Quarantine,  Additional Testing
$\}$.
Following \citet{kiyani2025decision}, we use the
baseline matrix $\Lambda_0$, transformed by $\ell(a,y)=\max_{a\in\mathcal{A},y\in\mathcal{Y}}u(a,y) - u(a,y)$ with their utility function $u$.
To probe higher-stakes regimes in which incorrect interventions are substantially more costly, we additionally consider 
a variant $\Lambda_1$ where critical mistakes have higher losses:
$$
\Lambda_0=
\begin{pmatrix}
0 & 8 & 8 & 6\\
10 & 0 & 7 & 3\\
10 & 7 & 0 & 2\\
9 & 6 & 6 & 0
\end{pmatrix},
\quad
\Lambda_1=
\begin{pmatrix}
0 & 8 & 8 & 6\\
100 & 0 & 70 & 3\\
100 & 70 & 0 & 2\\
90 & 60 & 60 & 0
\end{pmatrix}.
$$

This variant multiplies the loss of severe mismatches (e.g., choosing \textit{No Action} for a pathological label) by a factor of 10, making it especially important to account for potentially out-of-set labels.
This is the regime where ROCP’s out-of-set robustness can differ most from in-set max-min policies.

We vary the miscoverage level $\alpha$ to study its effect on performance. 
In \Cref{fig:covid-lambda0}, which uses the baseline loss matrix, \texttt{ROCP} attains worst-case risk that are close to \texttt{RAC} across all $\alpha$. While \texttt{ROCP} has slightly higher average realized loss than \texttt{RAC}, it consistently yields lower critical mistake rates. 
The best-response rule achieves the lowest average realized loss, but it 
has a much higher critical mistake rates. 

Finally, for decision-agnostic conformal set constructions (\texttt{LAS}/\texttt{APS}/\texttt{SOCOP}), applying $a_{\text{ROCP}}$ or $a_{\text{RA}}$ yields nearly identical performance, indicating that the dominant effect there comes from the set itself rather than the downstream policy. This close agreement is expected: under the baseline matrix there is no extreme penalty for severe mismatches, so the additional out-of-set term in the robust objective $L_{C(x)}(a;\alpha)$ (which is down-weighted by $\alpha$) has limited influence, and the resulting decision rule is close to the in-set max--min behavior of \texttt{RAC}.

The picture changes under the 10$\times$ severe-mismatch penalty matrix in \Cref{fig:covid-lambda1}. In this higher-stakes regime, \texttt{RAC} becomes brittle as $\alpha$ grows: both its worst-case certificate and realized loss deteriorate sharply at larger $\alpha$, whereas \texttt{ROCP} remains stable and achieves substantially lower worst-case risk and realized loss, nearly matching the best-response baseline on realized loss. Moreover, \texttt{ROCP} almost eliminates critical mistakes for all critical labels while \texttt{RAC} remains similar to the baseline-matrix setting, highlighting the benefit of explicitly accounting for the $\alpha$ fraction of out-of-set mass in the objective. This effect is also reflected in the conformal baselines: for $\alpha\ge 0.05$, applying $a_{\text{ROCP}}$ to \texttt{LAS}/\texttt{APS}/\texttt{SOCOP} sets consistently yields lower worst-case risk certificates and realized losses than applying the in-set max--min rule $a_{\text{RA}}$ to the same sets.

\subsection{Decision-making in an "autonomous driving"-like setting}
We consider a toy autonomous driving decision experiment built from the BDD100K driving dataset \citep{yu2020bdd100k}.
Our goal is to stress-test decision-making under set-valued uncertainty using a \emph{black-box} probabilistic model $f_x$
constructed from a pretrained  \texttt{YOLO11} detector \citep{yolo11_ultralytics}. 
Each image $x$ is mapped to a hazard label $Y=(Y_a,Y_\ell,Y_r)\in\{0,1\}^3$, where $Y_a$ indicates an occupied \emph{ahead-close} region (person or vehicle), and $Y_\ell$/$Y_r$ indicate a \emph{left-close}/\emph{right-close} nearby vehicle. The precise region-of-interest definitions, label construction from BDD annotations, the construction of $f_x$, and the definitions of the actions and loss
are deferred to \Cref{app:bdd-toy}.


\begin{figure}[t]
    \centering
    \includegraphics[width=0.8\linewidth]{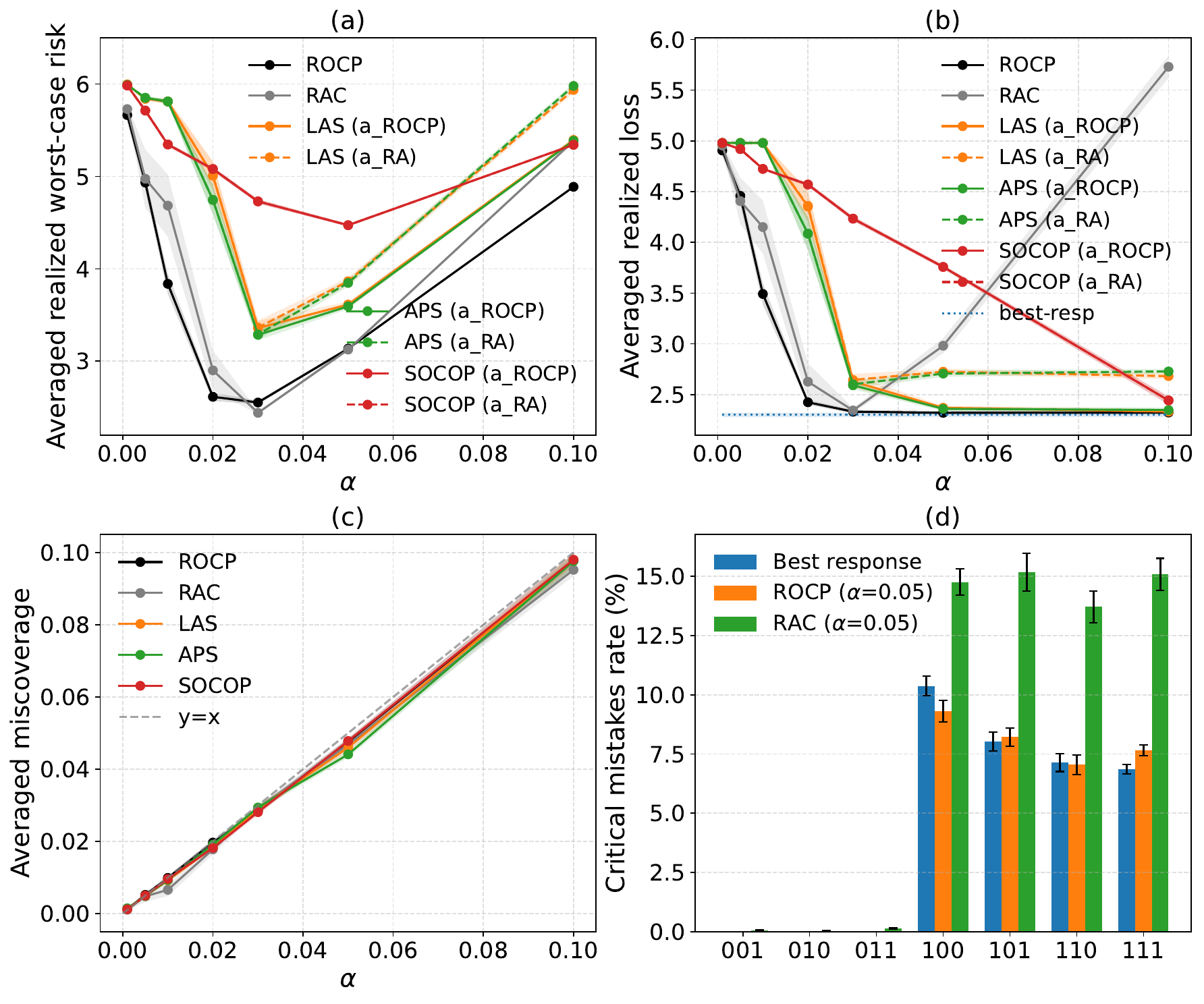}
    \caption{Toy autonomous driving experiment. (a) average realized worst-case risk certificate; (b) average realized loss; (c) empirical miscoverage; and (d) a critical mistake is defined as selecting an action that incurs the collision penalty (i.e., loss at least $M$) in that state.
    In (d), the x-axis labels \texttt{001}, \texttt{010}, $\ldots$, \texttt{111} denote the 3-bit hazard state $y=(y_a,y_\ell,y_r)$. Results are averaged over 20 random splits; error bars show $\pm 1$ standard error.}

    \label{fig:bdd-eval}

\end{figure}
As shown in \Cref{fig:bdd-eval}, \texttt{ROCP} matches or outperforms all baselines in both worst-case risk certificate and realized loss across all $\alpha$ values.
In particular, for the decision-agnostic conformal set constructions, pairing the same prediction sets with our robust decision rule $a_{\text{ROCP}}$ consistently improves performance over the in-set max--min rule $a_{\text{RA}}$ once $\alpha\ge 0.03$, yielding lower worst-case certificates and lower realized losses. Notably, the realized loss of \texttt{RAC} begins to increase after $\alpha\approx 0.03$, indicating that the max--min rule can become brittle when miscoverage is non-negligible in this setting, whereas \texttt{ROCP} remains stable.

Moreover, \texttt{ROCP} achieves lower critical mistake rates than \texttt{RAC} across hazardous states, while its realized loss close to it. This highlights the potential 
benefits of our method.

\section{Discussion}
Considering decision-making through the lens of expected loss minimization,
we developed a decision-theoretic interface between conformal prediction sets and downstream action selection.
An extension worth
pursuing could be to 
incorporate group-conditional, label-conditional, or localized guarantees could further reduce brittle behavior on structured subpopulations when such guarantees are statistically feasible. 

\section*{Acknowledgements}
This work was supported in part by the US NSF, ARO, AFOSR, ONR, the Simons Foundation and the Sloan Foundation.
 The authors thank Hamed Hassani, Shayan Kiyani, and Aaron Roth for helpful discussions about the work \cite{kiyani2025decision}.

\bibliography{ref}

\begin{thebibliography}{59}
\providecommand{\natexlab}[1]{#1}
\providecommand{\url}[1]{\texttt{#1}}
\expandafter\ifx\csname urlstyle\endcsname\relax
  \providecommand{\doi}[1]{doi: #1}\else
  \providecommand{\doi}{doi: \begingroup \urlstyle{rm}\Url}\fi

\bibitem[Angelopoulos et~al.(2020)Angelopoulos, Bates, Malik, and Jordan]{angelopoulos2020uncertainty}
A.~Angelopoulos, S.~Bates, J.~Malik, and M.~I. Jordan.
\newblock Uncertainty sets for image classifiers using conformal prediction.
\newblock \emph{arXiv preprint arXiv:2009.14193}, 2020.

\bibitem[Angelopoulos et~al.(2021)Angelopoulos, Bates, Cand{\`e}s, Jordan, and Lei]{angelopoulos2021learn}
A.~N. Angelopoulos, S.~Bates, E.~J. Cand{\`e}s, M.~I. Jordan, and L.~Lei.
\newblock Learn then test: Calibrating predictive algorithms to achieve risk control.
\newblock \emph{arXiv preprint arXiv:2110.01052}, 2021.

\bibitem[Angelopoulos et~al.(2022)Angelopoulos, Bates, Fisch, Lei, and Schuster]{angelopoulos2022conformal}
A.~N. Angelopoulos, S.~Bates, A.~Fisch, L.~Lei, and T.~Schuster.
\newblock Conformal risk control.
\newblock \emph{arXiv preprint arXiv:2208.02814}, 2022.

\bibitem[Angelopoulos et~al.(2023)Angelopoulos, Bates, et~al.]{angelopoulos2023conformal}
A.~N. Angelopoulos, S.~Bates, et~al.
\newblock Conformal prediction: A gentle introduction.
\newblock \emph{Foundations and Trends{\textregistered} in Machine Learning}, 16\penalty0 (4):\penalty0 494--591, 2023.

\bibitem[Balder(1985)]{balder1985elimination}
E.~Balder.
\newblock Elimination of randomization in statistical decision theory reconsidered.
\newblock \emph{Journal of multivariate analysis}, 16\penalty0 (2):\penalty0 260--264, 1985.

\bibitem[Barber et~al.(2020)Barber, Cand{\`e}s, Ramdas, and Tibshirani]{barber2020limits}
R.~F. Barber, E.~J. Cand{\`e}s, A.~Ramdas, and R.~J. Tibshirani.
\newblock The limits of distribution-free conditional predictive inference, 2020.

\bibitem[Blot et~al.(2024)Blot, Angelopoulos, Jordan, and Brunel]{blot2024automatically}
V.~Blot, A.~N. Angelopoulos, M.~I. Jordan, and N.~J. Brunel.
\newblock Automatically adaptive conformal risk control.
\newblock \emph{arXiv preprint arXiv:2406.17819}, 2024.

\bibitem[Chan et~al.(2024)Chan, Delage, and Lin]{chan2024conformal}
T.~Chan, E.~Delage, and B.~Lin.
\newblock Conformal inverse optimization for adherence-aware prescriptive analytics.
\newblock \emph{Available at SSRN}, 2024.

\bibitem[Chan and Kaw(2020)]{chan2020inverse}
T.~C. Chan and N.~Kaw.
\newblock Inverse optimization for the recovery of constraint parameters.
\newblock \emph{European Journal of Operational Research}, 282\penalty0 (2):\penalty0 415--427, 2020.

\bibitem[Chan et~al.(2023)Chan, Mahmood, and Zhu]{chan2023inverse}
T.~C. Chan, R.~Mahmood, and I.~Y. Zhu.
\newblock Inverse optimization: Theory and applications.
\newblock \emph{Operations Research}, 2023.

\bibitem[Chowdhury et~al.(2020)Chowdhury, Rahman, Khandakar, Mazhar, Kadir, Mahbub, Islam, Khan, Iqbal, Al~Emadi, et~al.]{chowdhury2020can}
M.~E. Chowdhury, T.~Rahman, A.~Khandakar, R.~Mazhar, M.~A. Kadir, Z.~B. Mahbub, K.~R. Islam, M.~S. Khan, A.~Iqbal, N.~Al~Emadi, et~al.
\newblock Can ai help in screening viral and covid-19 pneumonia?
\newblock \emph{Ieee Access}, 8:\penalty0 132665--132676, 2020.

\bibitem[Cortes-Gomez et~al.(2024)Cortes-Gomez, Pati{\~n}o, Byun, Wu, Horvitz, and Wilder]{cortes2024decision}
S.~Cortes-Gomez, C.~Pati{\~n}o, Y.~Byun, S.~Wu, E.~Horvitz, and B.~Wilder.
\newblock Decision-focused uncertainty quantification.
\newblock \emph{arXiv preprint arXiv:2410.01767}, 2024.

\bibitem[Danskin(2012)]{danskin2012theory}
J.~M. Danskin.
\newblock \emph{The theory of max-min and its application to weapons allocation problems}, volume~5.
\newblock Springer Science \& Business Media, 2012.

\bibitem[Dvoretzky et~al.(1951)Dvoretzky, Wald, and Wolfowitz]{dvoretzky1951elimination}
A.~Dvoretzky, A.~Wald, and J.~Wolfowitz.
\newblock Elimination of randomization in certain statistical decision procedures and zero-sum two-person games.
\newblock \emph{Ann. Math. Statist.}, 22\penalty0 (4):\penalty0 1--21, 1951.

\bibitem[Elmachtoub and Grigas(2022)]{elmachtoub2022smart}
A.~N. Elmachtoub and P.~Grigas.
\newblock Smart “predict, then optimize”.
\newblock \emph{Management Science}, 68\penalty0 (1):\penalty0 9--26, 2022.

\bibitem[Foygel~Barber et~al.(2021)Foygel~Barber, Candes, Ramdas, and Tibshirani]{foygel2021limits}
R.~Foygel~Barber, E.~J. Candes, A.~Ramdas, and R.~J. Tibshirani.
\newblock The limits of distribution-free conditional predictive inference.
\newblock \emph{Information and Inference: A Journal of the IMA}, 10\penalty0 (2):\penalty0 455--482, 2021.

\bibitem[Gibbs et~al.(2025)Gibbs, Cherian, and Cand{\`e}s]{gibbs2025conformal}
I.~Gibbs, J.~J. Cherian, and E.~J. Cand{\`e}s.
\newblock Conformal prediction with conditional guarantees.
\newblock \emph{Journal of the Royal Statistical Society Series B: Statistical Methodology}, page qkaf008, 2025.

\bibitem[Guan(2023)]{guan2023localized}
L.~Guan.
\newblock Localized conformal prediction: A generalized inference framework for conformal prediction.
\newblock \emph{Biometrika}, 110\penalty0 (1):\penalty0 33--50, 2023.

\bibitem[Jocher and Qiu(2024)]{yolo11_ultralytics}
G.~Jocher and J.~Qiu.
\newblock {Ultralytics YOLO11}, 2024.
\newblock URL \url{https://github.com/ultralytics/ultralytics}.

\bibitem[Johnstone and Cox(2021)]{johnstone2021conformal}
C.~Johnstone and B.~Cox.
\newblock Conformal uncertainty sets for robust optimization.
\newblock In \emph{Conformal and Probabilistic Prediction and Applications}, pages 72--90. PMLR, 2021.

\bibitem[Joshi et~al.(2025{\natexlab{a}})Joshi, Kiyani, Pappas, Dobriban, and Hassani]{joshi2025conformal}
S.~Joshi, S.~Kiyani, G.~Pappas, E.~Dobriban, and H.~Hassani.
\newblock Conformal inference under high-dimensional covariate shifts via likelihood-ratio regularization.
\newblock \emph{arXiv preprint arXiv:2502.13030}, 2025{\natexlab{a}}.

\bibitem[Joshi et~al.(2025{\natexlab{b}})Joshi, Sun, Hassani, and Dobriban]{joshi2025multirisk}
S.~Joshi, Y.~Sun, H.~Hassani, and E.~Dobriban.
\newblock Multirisk: Multiple risk control via iterative score thresholding.
\newblock \emph{arXiv preprint arXiv:2512.24587}, 2025{\natexlab{b}}.

\bibitem[Keith and Ahner(2021)]{keith2021survey}
A.~J. Keith and D.~K. Ahner.
\newblock A survey of decision making and optimization under uncertainty.
\newblock \emph{Annals of Operations Research}, 300\penalty0 (2):\penalty0 319--353, 2021.

\bibitem[Kiyani et~al.(2025)Kiyani, Pappas, Roth, and Hassani]{kiyani2025decision}
S.~Kiyani, G.~J. Pappas, A.~Roth, and H.~Hassani.
\newblock Decision theoretic foundations for conformal prediction: Optimal uncertainty quantification for risk-averse agents.
\newblock In \emph{Forty-second International Conference on Machine Learning}, 2025.
\newblock URL \url{https://openreview.net/forum?id=Ukjl86EsIk}.

\bibitem[Kuratowski and Ryll-Nardzewski(1965)]{kuratowski1965general}
K.~Kuratowski and C.~Ryll-Nardzewski.
\newblock A general theorem on selectors.
\newblock \emph{Bull. Acad. Polon. Sci. S{\'e}r. Sci. Math. Astronom. Phys}, 13\penalty0 (6):\penalty0 397--403, 1965.

\bibitem[Lehmann and Casella(1998)]{lehmann1998theory}
E.~Lehmann and G.~Casella.
\newblock Theory of point estimation.
\newblock \emph{Springer Texts in Statistics}, 1998.

\bibitem[Lei and Wasserman(2014)]{lei2014distribution}
J.~Lei and L.~Wasserman.
\newblock Distribution-free prediction bands for non-parametric regression.
\newblock \emph{Journal of the Royal Statistical Society: Series B (Statistical Methodology)}, 76\penalty0 (1):\penalty0 71--96, 2014.

\bibitem[Lei et~al.(2013)Lei, Robins, and Wasserman]{lei2013distribution}
J.~Lei, J.~Robins, and L.~Wasserman.
\newblock Distribution-free prediction sets.
\newblock \emph{Journal of the American Statistical Association}, 108\penalty0 (501):\penalty0 278--287, 2013.

\bibitem[Lei et~al.(2018)Lei, G’Sell, Rinaldo, Tibshirani, and Wasserman]{lei2018distribution}
J.~Lei, M.~G’Sell, A.~Rinaldo, R.~J. Tibshirani, and L.~Wasserman.
\newblock Distribution-free predictive inference for regression.
\newblock \emph{Journal of the American Statistical Association}, 113\penalty0 (523):\penalty0 1094--1111, 2018.

\bibitem[Lekeufack et~al.(2024)Lekeufack, Angelopoulos, Bajcsy, Jordan, and Malik]{lekeufack2024conformal}
J.~Lekeufack, A.~N. Angelopoulos, A.~Bajcsy, M.~I. Jordan, and J.~Malik.
\newblock Conformal decision theory: Safe autonomous decisions from imperfect predictions.
\newblock In \emph{2024 IEEE International Conference on Robotics and Automation (ICRA)}, pages 11668--11675. IEEE, 2024.

\bibitem[Lindemann et~al.(2023)Lindemann, Cleaveland, Shim, and Pappas]{lindemann2023safe}
L.~Lindemann, M.~Cleaveland, G.~Shim, and G.~J. Pappas.
\newblock Safe planning in dynamic environments using conformal prediction.
\newblock \emph{IEEE Robotics and Automation Letters}, 2023.

\bibitem[Noarov et~al.(2023)Noarov, Ramalingam, Roth, and Xie]{noarov2023high}
G.~Noarov, R.~Ramalingam, A.~Roth, and S.~Xie.
\newblock High-dimensional prediction for sequential decision making.
\newblock \emph{arXiv preprint arXiv:2310.17651}, 2023.

\bibitem[Papadopoulos et~al.(2002)Papadopoulos, Proedrou, Vovk, and Gammerman]{papadopoulos2002inductive}
H.~Papadopoulos, K.~Proedrou, V.~Vovk, and A.~Gammerman.
\newblock Inductive confidence machines for regression.
\newblock In \emph{European Conference on Machine Learning}, pages 345--356. Springer, 2002.

\bibitem[Park et~al.(2022)Park, Dobriban, Lee, and Bastani]{park2021pac}
S.~Park, E.~Dobriban, I.~Lee, and O.~Bastani.
\newblock {PAC} prediction sets under covariate shift.
\newblock In \emph{International Conference on Learning Representations}, 2022.

\bibitem[Patel et~al.(2024)Patel, Rayan, and Tewari]{patel2024conformal}
Y.~P. Patel, S.~Rayan, and A.~Tewari.
\newblock Conformal contextual robust optimization.
\newblock In \emph{International Conference on Artificial Intelligence and Statistics}, pages 2485--2493. PMLR, 2024.

\bibitem[Rahman et~al.(2021)Rahman, Khandakar, Qiblawey, Tahir, Kiranyaz, Kashem, Islam, Al~Maadeed, Zughaier, Khan, et~al.]{rahman2021exploring}
T.~Rahman, A.~Khandakar, Y.~Qiblawey, A.~Tahir, S.~Kiranyaz, S.~B.~A. Kashem, M.~T. Islam, S.~Al~Maadeed, S.~M. Zughaier, M.~S. Khan, et~al.
\newblock Exploring the effect of image enhancement techniques on covid-19 detection using chest x-ray images.
\newblock \emph{Computers in biology and medicine}, 132:\penalty0 104319, 2021.

\bibitem[Rockafellar and Wets(1998)]{rockafellar1998variational}
R.~T. Rockafellar and R.~J. Wets.
\newblock \emph{Variational analysis}.
\newblock Springer, 1998.

\bibitem[Romano et~al.(2019)Romano, Patterson, and Candes]{romano2019conformalized}
Y.~Romano, E.~Patterson, and E.~Candes.
\newblock Conformalized quantile regression.
\newblock \emph{Advances in neural information processing systems}, 32, 2019.

\bibitem[Romano et~al.(2020)Romano, Sesia, and Candes]{romano2020classification}
Y.~Romano, M.~Sesia, and E.~Candes.
\newblock Classification with valid and adaptive coverage.
\newblock \emph{Advances in Neural Information Processing Systems}, 33:\penalty0 3581--3591, 2020.

\bibitem[Sadinle et~al.(2019)Sadinle, Lei, and Wasserman]{Sadinle2019}
M.~Sadinle, J.~Lei, and L.~Wasserman.
\newblock {Least Ambiguous Set-Valued Classifiers With Bounded Error Levels}.
\newblock \emph{Journal of the American Statistical Association}, 114\penalty0 (525):\penalty0 223--234, 2019.

\bibitem[Saunders et~al.(1999)Saunders, Gammerman, and Vovk]{saunders1999transduction}
C.~Saunders, A.~Gammerman, and V.~Vovk.
\newblock Transduction with confidence and credibility.
\newblock In \emph{IJCAI}, 1999.

\bibitem[Scheffe and Tukey(1945)]{scheffe1945non}
H.~Scheffe and J.~W. Tukey.
\newblock Non-parametric estimation. i. validation of order statistics.
\newblock \emph{The Annals of Mathematical Statistics}, 16\penalty0 (2):\penalty0 187--192, 1945.

\bibitem[Sierpi{\'n}ski(1922)]{sierpinski1922fonctions}
W.~Sierpi{\'n}ski.
\newblock Sur les fonctions d'ensemble additives et continues.
\newblock \emph{Fundamenta Mathematicae}, 3\penalty0 (1):\penalty0 240--246, 1922.

\bibitem[Szegedy et~al.(2015)Szegedy, Liu, Jia, Sermanet, Reed, Anguelov, Erhan, Vanhoucke, and Rabinovich]{szegedy2015going}
C.~Szegedy, W.~Liu, Y.~Jia, P.~Sermanet, S.~Reed, D.~Anguelov, D.~Erhan, V.~Vanhoucke, and A.~Rabinovich.
\newblock Going deeper with convolutions.
\newblock In \emph{Proceedings of the IEEE conference on computer vision and pattern recognition}, pages 1--9, 2015.

\bibitem[Szegedy et~al.(2016)Szegedy, Vanhoucke, Ioffe, Shlens, and Wojna]{szegedy2016rethinking}
C.~Szegedy, V.~Vanhoucke, S.~Ioffe, J.~Shlens, and Z.~Wojna.
\newblock Rethinking the inception architecture for computer vision.
\newblock In \emph{Proceedings of the IEEE conference on computer vision and pattern recognition}, pages 2818--2826, 2016.

\bibitem[Tukey(1947)]{tukey1947non}
J.~W. Tukey.
\newblock Non-parametric estimation ii. statistically equivalent blocks and tolerance regions--the continuous case.
\newblock \emph{The Annals of Mathematical Statistics}, pages 529--539, 1947.

\bibitem[Vovk(2012)]{vovk2012conditional}
V.~Vovk.
\newblock Conditional validity of inductive conformal predictors.
\newblock In \emph{Asian conference on machine learning}, pages 475--490. PMLR, 2012.

\bibitem[Vovk et~al.(2005)Vovk, Gammerman, and Shafer]{vovk2005algorithmic}
V.~Vovk, A.~Gammerman, and G.~Shafer.
\newblock \emph{Algorithmic learning in a random world}, volume~29.
\newblock Springer, 2005.

\bibitem[Vovk et~al.(1999)Vovk, Gammerman, and Saunders]{vovk1999machine}
V.~Vovk, A.~Gammerman, and C.~Saunders.
\newblock Machine-learning applications of algorithmic randomness.
\newblock In \emph{International Conference on Machine Learning}, 1999.

\bibitem[Wald(1943)]{Wald1943}
A.~Wald.
\newblock {An Extension of Wilks' Method for Setting Tolerance Limits}.
\newblock \emph{The Annals of Mathematical Statistics}, 14\penalty0 (1):\penalty0 45--55, 1943.
\newblock ISSN 0003-4851.
\newblock \doi{10.1214/aoms/1177731491}.

\bibitem[Wald(1945)]{wald1945statistical}
A.~Wald.
\newblock Statistical decision functions which minimize the maximum risk.
\newblock \emph{Annals of Mathematics}, pages 265--280, 1945.

\bibitem[Wald(1949)]{wald1949statistical}
A.~Wald.
\newblock Statistical decision functions.
\newblock \emph{The Annals of Mathematical Statistics}, 20\penalty0 (2):\penalty0 165--205, 1949.

\bibitem[Wang et~al.(2025)Wang, Sun, and Dobriban]{wang2025singleton}
T.~Wang, Y.~Sun, and E.~Dobriban.
\newblock Singleton-optimized conformal prediction.
\newblock \emph{arXiv preprint arXiv:2509.24095}, 2025.

\bibitem[Wilks(1941)]{Wilks1941}
S.~S. Wilks.
\newblock {Determination of Sample Sizes for Setting Tolerance Limits}.
\newblock \emph{The Annals of Mathematical Statistics}, 12\penalty0 (1):\penalty0 91--96, 1941.
\newblock ISSN 0003-4851.
\newblock \doi{10.1214/aoms/1177731788}.

\bibitem[Yeh et~al.(2024)Yeh, Christianson, Wu, Wierman, and Yue]{yeh2024end}
C.~Yeh, N.~Christianson, A.~Wu, A.~Wierman, and Y.~Yue.
\newblock End-to-end conformal calibration for optimization under uncertainty.
\newblock \emph{arXiv preprint arXiv:2409.20534}, 2024.

\bibitem[Yu et~al.(2020)Yu, Chen, Wang, Xian, Chen, Liu, Madhavan, and Darrell]{yu2020bdd100k}
F.~Yu, H.~Chen, X.~Wang, W.~Xian, Y.~Chen, F.~Liu, V.~Madhavan, and T.~Darrell.
\newblock Bdd100k: A diverse driving dataset for heterogeneous multitask learning.
\newblock In \emph{Proceedings of the IEEE/CVF conference on computer vision and pattern recognition}, pages 2636--2645, 2020.

\bibitem[Zecchin and Simeone(2024{\natexlab{a}})]{zecchin2024adaptive}
M.~Zecchin and O.~Simeone.
\newblock Adaptive learn-then-test: Statistically valid and efficient hyperparameter selection.
\newblock \emph{arXiv preprint arXiv:2409.15844}, 2024{\natexlab{a}}.

\bibitem[Zecchin and Simeone(2024{\natexlab{b}})]{zecchin2024localized}
M.~Zecchin and O.~Simeone.
\newblock Localized adaptive risk control.
\newblock \emph{arXiv preprint arXiv:2405.07976}, 2024{\natexlab{b}}.

\bibitem[Zhao et~al.(2021)Zhao, Kim, Sahoo, Ma, and Ermon]{zhao2021calibrating}
S.~Zhao, M.~Kim, R.~Sahoo, T.~Ma, and S.~Ermon.
\newblock Calibrating predictions to decisions: A novel approach to multi-class calibration.
\newblock \emph{Advances in Neural Information Processing Systems}, 34:\penalty0 22313--22324, 2021.

\end{thebibliography}
\bibliographystyle{plainnat}

\newpage
\appendix
\onecolumn

\section{Additional Related Work and Background}
\label{app:related-work}
{\bf Conformal prediction and prediction sets.}
Prediction sets have classical roots in statistics \citep{Wilks1941,Wald1943,scheffe1945non,tukey1947non} and were developed into the modern conformal prediction framework
starting with 
\citet{saunders1999transduction,vovk1999machine,papadopoulos2002inductive,vovk2005algorithmic}.
With the rise of modern ML, conformal prediction has become a standard tool for distribution-free uncertainty quantification across tasks including classification and regression \citep{lei2018distribution,romano2020classification,romano2019conformalized,park2021pac,angelopoulos2020uncertainty,angelopoulos2023conformal}.
Our work uses conformal prediction only for its coverage guarantee; the key question we address is how to turn that guarantee into \emph{optimal downstream decisions} and \emph{decision-optimal set construction}.

{\bf Conformal methods for risk control and decision-aware sets.}
A growing literature goes beyond coverage and uses conformal ideas to control general risk measures \citep{angelopoulos2022conformal,angelopoulos2021learn,lindemann2023safe,lekeufack2024conformal,zecchin2024adaptive,blot2024automatically,zecchin2024localized,cortes2024decision,joshi2025multirisk}.
These works typically focus on guaranteeing that a chosen risk functional of the predictor is below a target level, often for a fixed decision rule.
In contrast, our goal is to \emph{jointly} characterize the optimal decision rule induced by coverage and the optimal prediction sets for that induced rule when the downstream objective is expected loss under worst-case distributions consistent with coverage.

{\bf Calibration and best-response baselines.}
When probabilistic forecasts are calibrated, best-responding to the predictive distribution is optimal for expectation-maximizing agents \citep{zhao2021calibrating,noarov2023high}.
We include calibrated best-response baselines in our experiments to highlight the practical trade-off: committing to a single action can deliver strong average performance when the model is reliable, but can incur catastrophic errors when tail events are misestimated.

\section{Proofs}
\label{app:proofs}
{\bf Measurability conventions.}
In the main text we suppress measure-theoretic details; throughout this appendix we make them explicit.
Let $(\mathcal X,\mathcal F)$, $(\mathcal Y,\mathcal G)$, and $(\mathcal A,\mathcal H)$ be standard Borel measurable spaces.
All random variables are defined on a common probability space and take values in $\mathcal X$ and $\mathcal Y$.
The loss $\ell:\mathcal A\times\mathcal Y\to[0,\infty)$ is assumed $\mathcal H\otimes\mathcal G$-measurable, and regular conditional laws exist.
Whenever we write $Q(S)$ or $\Pr(Y\in S)$, we implicitly assume $S\in\mathcal G$ is measurable.
For set-valued maps $C:\mathcal X\to 2^{\mathcal Y}$ appearing inside probabilities such as $\Pr\{Y\in C(X)\}$, we assume $C$ is measurable in the sense that its graph
$\{(x,y):y\in C(x)\}$ belongs to $\mathcal F\otimes\mathcal G$, so the event $\{Y\in C(X)\}$ is measurable.
We also tacitly restrict to policies $\pi$ for which $x\mapsto \pi(C(x))$ is $\mathcal H$-measurable, ensuring $\ell(\pi(C(X)),Y)$ is measurable.

{\bf Proof of Lemma~\ref{lem:L-formula}.}
Fix any measurable $S\subseteq\mathcal Y$ and $a\in\mathcal A$. 

{\bf Case 1:} $S = \mathcal{Y}$. 
The constraint $Q(\mathcal{Y}) \ge 1-\alpha$ holds for any probability measure $Q$. Since $\mathbb{E}_{Y \sim Q}[\ell(a,Y)] \le \sup_{y \in \mathcal{Y}} \ell(a,y)$, and conversely, taking $Q = \delta_{y^*}$ with $y^* = \arg\max_{y \in \mathcal{Y}} \ell(a,y)$ yields 
$L_{\mathcal{Y}}(a;\alpha) = \ell_{\mathcal{Y}}^{\mathrm{in}}(a)$, matching \eqref{eq:worst-case loss form} under the convention $\ell_{\mathcal{Y}}^{\mathrm{out}}(a) := \ell_{\mathcal{Y}}^{\mathrm{in}}(a)$.

{\bf Case 2:} $S \subsetneq\mathcal{Y}$. For any probability measure $Q$ on $\mathcal Y$ with $Q(S)\ge 1-\alpha$,
\begin{align*}
\mathbb E_{Y\sim Q}[\ell(a,Y)]
&= \int_S \ell(a,y)\,Q(\mathrm dy)\ +\ \int_{S^c}\ell(a,y)\,Q(\mathrm dy)\\
&\le Q(S)\,\sup_{y\in S}\ell(a,y)\ +\ (1-Q(S))\,\sup_{y\notin S}\ell(a,y)\\
&= Q(S)\,\ell_S^{\mathrm{in}}(a)\ +\ (1-Q(S))\,\ell_S^{\mathrm{out}}(a).
\end{align*}
The right-hand side is affine in $Q(S)\in[1-\alpha,1]$, hence
\[
\mathbb E_{Y\sim Q}[\ell(a,Y)]
\le
\max\Big\{\ell_S^{\mathrm{in}}(a),\ (1-\alpha)\ell_S^{\mathrm{in}}(a)+\alpha\,\ell_S^{\mathrm{out}}(a)\Big\}.
\]
Taking the supremum over all such $Q$ gives
\[
L_S(a;\alpha)\ \le\ \max\Big\{\ell_S^{\mathrm{in}}(a),\ (1-\alpha)\ell_S^{\mathrm{in}}(a)+\alpha\,\ell_S^{\mathrm{out}}(a)\Big\}.
\]

For the reverse inequality, fix $\varepsilon>0$. By definition of supremum, choose $y_{\mathrm{i}}\in S$ such that
$\ell(a,y_{\mathrm{i}})\ge \ell_S^{\mathrm{in}}(a)-\varepsilon$.
If $\ell_S^{\mathrm{out}}(a)\le \ell_S^{\mathrm{in}}(a)$, take $Q_\varepsilon=\delta_{y_{\mathrm{i}}}$; then $Q_\varepsilon(S)=1$ and
\[
\mathbb E_{Y\sim Q_\varepsilon}[\ell(a,Y)] \ge \ell_S^{\mathrm{in}}(a)-\varepsilon.
\]
If instead $\ell_S^{\mathrm{out}}(a)>\ell_S^{\mathrm{in}}(a)$, choose $y_{\mathrm{o}}\notin S$ such that
$\ell(a,y_{\mathrm{o}})\ge \ell_S^{\mathrm{out}}(a)-\varepsilon$, and take
$Q_\varepsilon=(1-\alpha)\delta_{y_{\mathrm{i}}}+\alpha\,\delta_{y_{\mathrm{o}}}$; then $Q_\varepsilon(S)=1-\alpha$ and
\[
\mathbb E_{Y\sim Q_\varepsilon}[\ell(a,Y)]
\ge (1-\alpha)(\ell_S^{\mathrm{in}}(a)-\varepsilon)+\alpha(\ell_S^{\mathrm{out}}(a)-\varepsilon)
= (1-\alpha)\ell_S^{\mathrm{in}}(a)+\alpha\,\ell_S^{\mathrm{out}}(a)-\varepsilon.
\]
In either case,
\[
L_S(a;\alpha)\ \ge\ \max\Big\{\ell_S^{\mathrm{in}}(a),\ (1-\alpha)\ell_S^{\mathrm{in}}(a)+\alpha\,\ell_S^{\mathrm{out}}(a)\Big\}-\varepsilon.
\]
Letting $\varepsilon\downarrow 0$ yields the matching lower bound. Combining with the upper bound,
\[
L_S(a;\alpha)
=
\max\Big\{\ell_S^{\mathrm{in}}(a),\ (1-\alpha)\ell_S^{\mathrm{in}}(a)+\alpha\,\ell_S^{\mathrm{out}}(a)\Big\}
=
\ell_S^{\mathrm{in}}(a)+\alpha\big(\ell_S^{\mathrm{out}}(a)-\ell_S^{\mathrm{in}}(a)\big)_+,
\]
as claimed.
\qed

{\bf Proof of Theorem~\ref{thm:main}.}
Fix any policy $\pi$ and write $a(x):=\pi(C(x))$.
We first show that for any fixed $\pi$,
\begin{equation}\label{eq:fixed-pi-sup}
\sup_{P\in\mathcal{P}_\alpha}\ \mathbb E_P\,\ell\big(\pi(C(X)),Y\big)
 = \sup_{x\in\mathcal X} L_{C(x)}\big(\pi(C(x));\alpha\big).
\end{equation}

\emph{Upper bound.}
Take any $P\in\mathcal{P}_\alpha$. By the definition of $\mathcal{P}_\alpha$, for $P_X$-a.e.\ $x$ the conditional law
$Q_x(\cdot):=P(Y\in\cdot\mid X=x)$ satisfies $Q_x(C(x))\ge 1-\alpha$.
Hence, by the definition \eqref{eq:wc-L} of $L_S(a;\alpha)$,
\[
\mathbb E_P\!\left[\ell\big(a(X),Y\big)\mid X=x\right]
\le L_{C(x)}\big(a(x);\alpha\big)
= L_{C(x)}\big(\pi(C(x));\alpha\big)
\quad\text{for $P_X$-a.e.\ }x.
\]
Taking expectations over $X$ gives
\[
\mathbb E_P\,\ell\big(\pi(C(X)),Y\big)
\le \mathbb E_{X\sim P_X}\!\left[L_{C(X)}\big(\pi(C(X));\alpha\big)\right]
\le \sup_{x\in\mathcal X} L_{C(x)}\big(\pi(C(x));\alpha\big).
\]
Since this holds for all $P\in\mathcal{P}_\alpha$, we obtain
\[
\sup_{P\in\mathcal{P}_\alpha}\ \mathbb E_P\,\ell\big(\pi(C(X)),Y\big)
\le \sup_{x\in\mathcal X} L_{C(x)}\big(\pi(C(x));\alpha\big).
\]

\emph{Lower bound.}
Fix $\varepsilon>0$ and choose $x_\varepsilon\in\mathcal X$ such that
\[
L_{C(x_\varepsilon)}\big(\pi(C(x_\varepsilon));\alpha\big)
\ \ge\
\sup_{x\in\mathcal X}L_{C(x)}\big(\pi(C(x));\alpha\big)-\varepsilon.
\]
Let $S_\varepsilon:=C(x_\varepsilon)$ and $a_\varepsilon:=\pi(S_\varepsilon)$, and write $
\ell^{\mathrm{in}}_\varepsilon:=\ell^{\mathrm{in}}_{S_\varepsilon}(a_\varepsilon)$, $
\ell^{\mathrm{out}}_\varepsilon:=\ell^{\mathrm{out}}_{S_\varepsilon}(a_\varepsilon) $.
Set $P_X=\delta_{x_\varepsilon}$.

\emph{Case 1: $S_\varepsilon\ne\mathcal Y$ and $\ell^{\mathrm{out}}_\varepsilon> \ell^{\mathrm{in}}_\varepsilon$.}
Choose $y_{\mathrm{i}}\in S_\varepsilon$ and $y_{\mathrm{o}}\notin S_\varepsilon$ such that
\[
\ell(a_\varepsilon,y_{\mathrm{i}})\ge \ell^{\mathrm{in}}_\varepsilon-\varepsilon,
\qquad
\ell(a_\varepsilon,y_{\mathrm{o}})\ge \ell^{\mathrm{out}}_\varepsilon-\varepsilon.
\]
Define $Y\mid X=x_\varepsilon \sim (1-\alpha)\delta_{y_{\mathrm{i}}}+\alpha\,\delta_{y_{\mathrm{o}}}$.
Then $P\in\mathcal{P}_\alpha$ and
\[
\mathbb E_P\,\ell\big(\pi(C(X)),Y\big)
\ge (1-\alpha)\ell^{\mathrm{in}}_\varepsilon+\alpha\,\ell^{\mathrm{out}}_\varepsilon-\varepsilon
= L_{S_\varepsilon}(a_\varepsilon;\alpha)-\varepsilon.
\]

\emph{Case 2: $S_\varepsilon=\mathcal Y$ or $\ell^{\mathrm{out}}_\varepsilon\le \ell^{\mathrm{in}}_\varepsilon$.}
Choose $y_{\mathrm{i}}\in S_\varepsilon$ such that
$\ell(a_\varepsilon,y_{\mathrm{i}})\ge \ell^{\mathrm{in}}_\varepsilon-\varepsilon$
and define $Y\mid X=x_\varepsilon\sim \delta_{y_{\mathrm{i}}}$.
Then $P\in\mathcal{P}_\alpha$ and
\[
\mathbb E_P\,\ell\big(\pi(C(X)),Y\big)
=\ell(a_\varepsilon,y_{\mathrm{i}})
\ge \ell^{\mathrm{in}}_\varepsilon-\varepsilon
= L_{S_\varepsilon}(a_\varepsilon;\alpha)-\varepsilon.
\]

Combining the two cases, we have constructed $P\in\mathcal{P}_\alpha$ such that
\[
\mathbb E_P\,\ell\big(\pi(C(X)),Y\big)\ \ge\ L_{S_\varepsilon}(a_\varepsilon;\alpha)-\varepsilon.
\]
By the choice of $x_\varepsilon$ and since $S_\varepsilon=C(x_\varepsilon)$, also
\[
L_{S_\varepsilon}(a_\varepsilon;\alpha)
=L_{C(x_\varepsilon)}\big(\pi(C(x_\varepsilon));\alpha\big)
\ge \sup_{x\in\mathcal X}L_{C(x)}\big(\pi(C(x));\alpha\big)-\varepsilon.
\]
Therefore
\[
\sup_{P\in\mathcal{P}_\alpha}\ \mathbb E_P\,\ell\big(\pi(C(X)),Y\big)
\ \ge\
\sup_{x\in\mathcal X}L_{C(x)}\big(\pi(C(x));\alpha\big)-2\varepsilon.
\]
Letting $\varepsilon\downarrow 0$ yields the reverse inequality in \eqref{eq:fixed-pi-sup}.

Now minimize over $\pi$. For any $\pi$,
\[
\sup_{x\in\mathcal X} L_{C(x)}\big(\pi(C(x));\alpha\big)\ \ge\ \sup_{x\in\mathcal X}\ \min_{a\in\mathcal A} L_{C(x)}(a;\alpha),
\]
hence
\[
\inf_{\pi}\ \sup_{P\in\mathcal{P}_\alpha}\ \mathbb E\,\ell\big(\pi(C(X)),Y\big)
=\inf_{\pi}\ \sup_{x\in\mathcal X} L_{C(x)}\big(\pi(C(x));\alpha\big)
\ \ge\ \sup_{x\in\mathcal X}\ \min_{a\in\mathcal A} L_{C(x)}(a;\alpha).
\]
Conversely, by the attainment assumption, for each set $S$ in the range $\mathrm{Im}(C)=\{C(x):x\in\mathcal X\}$ pick
$\pi^\star(S)\in\arg\min_{a\in\mathcal A}L_S(a;\alpha)$.
Then
\[
\sup_{x\in\mathcal X} L_{C(x)}\big(\pi^\star(C(x));\alpha\big)
=\sup_{x\in\mathcal X}\ \min_{a\in\mathcal A} L_{C(x)}(a;\alpha),
\]
so equality holds in \eqref{eq:value} and $\pi^\star$ satisfies \eqref{pi-star}.

Finally, under the additional attainment assumptions of the relevant suprema in $\ell^{\mathrm{in}}$ and $\ell^{\mathrm{out}}$, and attainment of the outer supremum over $x$, the above construction with $\varepsilon=0$ yields a worst-case $P^\star\in\mathcal{P}_\alpha$ concentrated at such an $x^\star$, with $Y\mid X=x^\star$ as stated in the theorem.
\qed
\begin{remark}[Measurable selection for \eqref{pi-star}]
\label{rem:measurable-pi-star}
The final step of the proof above picks, for each set $S$ in the range of $C$, some minimizer
$\pi^\star(S)\in\arg\min_{a\in\mathcal A}L_S(a;\alpha)$.
This pointwise choice does not automatically ensure that the composite map $x\mapsto \pi^\star(C(x))$ is $\mathcal H$-measurable, as required by our measurability conventions.
A sufficient condition is the following: assume that the function $f:\mathcal X\times\mathcal A\to\overline{\mathbb R}$ defined by
$f(x,a):=L_{C(x)}(a;\alpha)$ is a normal integrand in the sense of Definition~\ref{def:normal-integrand} and that, for every $x$, the minimum of $a\mapsto f(x,a)$ is attained.
Then the argmin correspondence
$\Gamma(x):=\arg\min_{a\in\mathcal A} f(x,a)$
is a measurable multifunction with nonempty closed values (see, e.g., \cite{rockafellar1998variational}, Thm.~14.37).
Since $(\mathcal A,\mathcal H)$ is standard Borel, the Kuratowski--Ryll--Nardzewski measurable selection theorem \citep{kuratowski1965general}
yields an $\mathcal H$-measurable selector $a^\star:\mathcal X\to\mathcal A$ such that $a^\star(x)\in\Gamma(x)$ for all $x$.
Taking $x\mapsto \pi^\star(C(x)):=a^\star(x)$ gives a measurable minimax-optimal policy satisfying \eqref{pi-star}.
\end{remark}

{\bf Proof of Proposition~\ref{prop:EL-pointwise}.}
Fix $x\in\mathcal X$ and $t\in(0,1]$. Let $C\subseteq\mathcal Y$ be measurable with
$\mathbb P(Y\in C\mid X=x)\ge t$, and fix any $a\in\mathcal A$.
Set
$s:=\ell_{C}^{\mathrm{in}}(a)=\sup_{y\in C}\ell(a,y)$.
Then $C\subseteq S_s(a):=\{y:\ell(a,y)\le s\}$, hence
\[
\mathbb P\big(\ell(a,Y)\le s\mid X=x\big)
=\mathbb P\big(Y\in S_s(a)\mid X=x\big)
\ge \mathbb P(Y\in C\mid X=x)\ge t.
\]
By definition of $Q_t^x(a)$, this implies $s\ge Q_t^x(a)$.

If $s<M(a)$, then $C^c\supseteq\{y:\ell(a,y)>s\}$, so $\ell_C^{\mathrm{out}}(a)=M(a)$.
Applying Lemma~\ref{lem:L-formula} with $\alpha=1-t$ gives
\[
L_C(a;1-t)
= s+(1-t)\big(\ell_C^{\mathrm{out}}(a)-s\big)_+
= s+(1-t)(M(a)-s)
= t\,s+(1-t)M(a)
\ge t\,Q_t^x(a)+(1-t)M(a).
\]
If $s=M(a)$, then $L_C(a;1-t)\ge s=M(a)\ge t\,Q_t^x(a)+(1-t)M(a)$ (since $Q_t^x(a)\le M(a)$).
Therefore, for every feasible $C$,
\[
R(C,1-t)=\min_{a\in\mathcal A}L_C(a;1-t)
\ \ge\
\min_{a\in\mathcal A}\Big\{t\,Q_t^x(a)+(1-t)M(a)\Big\}.
\]

Now let $a(x,t)$ be as in \eqref{eq:pointwise-optimal action}, define
$\theta(x,t):=Q_t^x(a(x,t))$, and set $C(x,t)$ as in \eqref{eq:pointwise-optimal set}.
Then $\mathbb P(Y\in C(x,t)\mid X=x)=\mathbb P(\ell(a(x,t),Y)\le \theta(x,t)\mid X=x)\ge t$, so $C(x,t)$ is feasible.
Moreover, since $C(x,t)=S_{\theta(x,t)}(a(x,t))$, we have
$\ell^{\mathrm{in}}_{C(x,t)}(a(x,t))\le \theta(x,t)$.
On the other hand, the argument above showed that for any feasible $C$,
$\ell_C^{\mathrm{in}}(a)\ge Q_t^x(a)$; applying this to $C(x,t)$ and $a(x,t)$ yields
\[
\ell^{\mathrm{in}}_{C(x,t)}(a(x,t))\ge Q_t^x(a(x,t))=\theta(x,t).
\]
Therefore $\ell^{\mathrm{in}}_{C(x,t)}(a(x,t))=\theta(x,t)$. If $\theta(x,t)<M(a(x,t))$, then $\ell^{\mathrm{out}}_{C(x,t)}(a(x,t))=M(a(x,t))$. If instead $\theta(x,t)=M(a(x,t))$, then trivially
$\ell^{\mathrm{out}}_{C(x,t)}(a(x,t))\le M(a(x,t))=\ell^{\mathrm{in}}_{C(x,t)}(a(x,t))$,
so $\big(\ell^{\mathrm{out}}_{C(x,t)}(a(x,t))-\ell^{\mathrm{in}}_{C(x,t)}(a(x,t))\big)_+=0$. Thus Lemma~\ref{lem:L-formula} yields
\[
L_{C(x,t)}(a(x,t);1-t)
= t\,\theta(x,t)+(1-t)M(a(x,t)).
\]
Consequently,
\[
R(C(x,t),1-t)
\le L_{C(x,t)}(a(x,t);1-t)
= \min_{a\in\mathcal A}\Big\{t\,Q_t^x(a)+(1-t)M(a)\Big\},
\]
which matches the lower bound, proving optimality and the claimed value.

Finally, if the suprema are attained, pick
$y_{\mathrm{i}}\in\arg\max_{y\in C(x,t)}\ell(a(x,t),y)$ and
$y_{\mathrm{o}}\in\arg\max_{y\in\mathcal Y}\ell(a(x,t),y)$ and take
$Q^\star=t\delta_{y_{\mathrm{i}}}+(1-t)\delta_{y_{\mathrm{o}}}$.
Then $Q^\star(C(x,t))\ge t$ and
$\mathbb E_{Y\sim Q^\star}[\ell(a(x,t),Y)]=t\,\theta(x,t)+(1-t)M(a(x,t))=L_{C(x,t)}(a(x,t);1-t)$,
so $Q^\star$ achieves the inner supremum in $L_{C(x,t)}(a(x,t);1-t)$.
\qed

{\bf Proof of Theorem~\ref{thm:pop-opt-general}.}
We first introduce the definition of normal integrands.
\begin{definition}[Normal integrands \citep{rockafellar1998variational}]
\label{def:normal-integrand}
Let $(\mathcal X,\mathcal F)$ be a measurable space and let $(\mathcal Z,\mathcal B(\mathcal Z))$ be a Polish space with its Borel $\sigma$-field.
A function $f:\mathcal X\times\mathcal Z\to\overline{\mathbb R}$ is called a \emph{normal integrand}
if its epigraphical mapping $S_f:\mathcal X\rightrightarrows \mathcal Z\times\mathbb R$,
defined by
\[
S_f(x) := \operatorname{epi} f(x,\cdot)
 := \big\{(z,\alpha)\in \mathcal Z\times\mathbb R:\ f(x,z)\le \alpha\big\},
\]
is closed-valued and measurable (i.e., its graph
$\{(x,z,\alpha): (z,\alpha)\in S_f(x)\}$ belongs to
$\mathcal F\otimes \mathcal B(\mathcal Z)\otimes \mathcal B(\mathbb R)$).
\end{definition}

Take $\mathcal Z=[0,1]$. We assume $(x,t)\mapsto V_x(t)$ in \eqref{eq:v_x(t)} is a normal integrand  in the sense of Definition~\ref{def:normal-integrand}.
Equivalently,\footnote{Under the measurability conventions, a sufficient condition is that
$\mathcal A$ is a compact metric space with $\mathcal H=\mathcal B(\mathcal A)$ and that
$a\mapsto \ell(a,y)$ is lower semicontinuous for every $y\in\mathcal Y$.}
\begin{enumerate}
\item[(i)] $(x,t)\mapsto V_x(t)$ is $\mathcal F\otimes\mathcal B([0,1])$-measurable;
\item[(ii)] for each $x$, the map $t\mapsto V_x(t)$ is lower semicontinuous on $[0,1]$.
\end{enumerate}

Recall the population problem \eqref{eq:pop-problem}:
\[
\operatorname{VAL}(\alpha)
=\inf_{\substack{t:\mathcal X\to[0,1]\ \text{measurable}\\ \mathbb E[t(X)]\ge 1-\alpha}}
\ \mathbb E\!\left[V_X\big(t(X)\big)\right].
\]
For $\beta \ge0$, define the dual function 
\[
\phi(\beta) := \beta(1-\alpha)\ +\ \mathbb E\!\left[\inf_{u\in[0,1]}\bigl\{V_X(u)-\beta u\bigr\}\right].
\]
and define the (set-valued) argmin correspondence
    \[
    \Gamma_\beta(x):=\arg\min_{u\in[0,1]}\bigl\{V_x(u)-\beta u\bigr\},
    \]
    and its extremal selectors
    \[
    g^+(x,\beta):=\max\Gamma_\beta(x),\qquad g^-(x,\beta):=\min\Gamma_\beta(x).
    \]
Since we assume that $(x,u)\mapsto V_x(u)$ is a normal integrand and $u\mapsto -\beta u$ is continuous,
the function $V_x(u)-\beta u$ is also a normal integrand; hence $\Gamma_\beta(x)=\arg\min_{u\in[0,1]} \{V_x(u)-\beta u\}$ is a measurable multifunction
with nonempty compact values (see, e.g., \cite{rockafellar1998variational}, Thm.~14.37), and therefore $g^{-}(\cdot,\beta)$ and $g^{+}(\cdot,\beta)$ are measurable (see, e.g., \cite{rockafellar1998variational}, Def. 14.1 or Ex.~14.51).

\begin{proof}[Proof of Theorem~\ref{thm:pop-opt-general}]
For $\beta\ge 0$, define the Lagrangian
\[
\mathcal L(t;\beta)
:=\mathbb E\!\left[V_X\big(t(X)\big)\right]+\beta\Big((1-\alpha)-\mathbb E[t(X)]\Big)
=\beta(1-\alpha)+\mathbb E\!\left[V_X\big(t(X)\big)-\beta\,t(X)\right].
\]
By Lemma~\ref{lem:interchange},
\[
\inf_{\substack{t:\mathcal X\to[0,1]\\ t\ \text{measurable}}}\ \mathcal L(t;\beta)
=\beta(1-\alpha)+\mathbb E\!\left[\inf_{u\in[0,1]}\{V_X(u)-\beta u\}\right]
=:\phi(\beta).
\]
Weak duality gives $\operatorname{VAL}(\alpha)\ge \sup_{\beta\ge 0}\phi(\beta)$.
By Lemma~\ref{lem:no-gap}, there is no duality gap, hence
\[
\operatorname{VAL}(\alpha)=\sup_{\beta\ge 0}\phi(\beta).
\]    
Moreover, the supremum is attained. For any $\beta\ge0$,
\[
\phi(\beta)
=\beta(1-\alpha)+\mathbb E\!\left[\inf_{u\in[0,1]}\{V_X(u)-\beta u\}\right]
\le \beta(1-\alpha)+\mathbb E[V_X(1)-\beta]
= \mathbb E[V_X(1)]-\alpha\beta,
\]
so $\phi(\beta)\to-\infty$ as $\beta\to\infty$ (since $\alpha>0$). Hence there exists $B<\infty$ such that
$\sup_{\beta\ge0}\phi(\beta) = \sup_{\beta\in[0,B]}\phi(\beta)$.
Next, define $\psi_x(\beta):=\inf_{u\in[0,1]}\{V_x(u)-\beta u\}$. For any $\beta,\beta'\ge0$ and any $x$,
\[
|\psi_x(\beta)-\psi_x(\beta')|\ \le\ |\beta-\beta'|
\]
since $u\in[0,1]$. Taking expectations and adding the linear term $\beta(1-\alpha)$ shows that $\phi$ is Lipschitz (hence continuous) on $[0,B]$.
Therefore, by compactness of $[0,B]$, there exists $\beta^\star\in[0,B]$ maximizing $\phi$. In addition, since each $\psi_x$ is concave and $1$-Lipschitz in $\beta$, hence the one-sided derivatives
$\psi'_{x,+}(\beta)$ and $\psi'_{x,-}(\beta)$ exist for all $\beta\ge 0$. By Danskin’s theorem \citep{danskin2012theory}, the one-sided derivatives satisfy
\[
\psi'_{x,+}(\beta) = -g^{+}(x,\beta),
\qquad
\psi'_{x,-}(\beta) = -g^{-}(x,\beta).
\]
Moreover, for any $h\neq 0$ we have
\[
\left|\frac{\psi_x(\beta+h)-\psi_x(\beta)}{h}\right|\le 1,
\]
so by dominated convergence we can obtain the one-sided derivatives of $\phi$:
\[
\phi'_{+}(\beta)
=(1-\alpha)+\mathbb E\big[\psi'_{X,+}(\beta)\big]
=(1-\alpha)-\mathbb E\big[g^{+}(X,\beta)\big],
\]
\[
\phi'_{-}(\beta)
=(1-\alpha)+\mathbb E\big[\psi'_{X,-}(\beta)\big]
=(1-\alpha)-\mathbb E\big[g^{-}(X,\beta)\big].
\]
Since $\phi$ is concave on $[0,\infty)$ and $\beta^\star$ maximizes $\phi$ if and only if $0 \in \partial \phi(\beta^\star)$, where the superdifferential $\partial \phi(\beta)$ is the closed interval $[\phi'_{+}(\beta), \phi'_{-}(\beta)]$. Thus if $\beta^\star>0$ (interior maximizer) then
$\phi'_{+}(\beta^\star)\le 0\le \phi'_{-}(\beta^\star)$,
which is equivalent to the interval condition \eqref{eq:interval condition}; For the boundary case $\beta^\star=0$, view $\phi$ as an extended-real concave function on $\mathbb R$ by setting $\phi(\beta):=-\infty$ for $\beta<0$.
Then $\partial \phi(0)=[\phi'_{+}(0),+\infty)$, so the optimality condition $0\in\partial\phi(0)$ reduces to $\phi'_{+}(0)\le 0$,
i.e.\ $1-\alpha\le \mathbb E[g^{+}(X,0)]$.

Next, we construct a primal optimizer. If $\beta^\star=0$,
let $t^\star(x):=g^+(x,0)\in\Gamma_0(x)$.
Since in this case $1-\alpha\le \mathbb E[g^{+}(X,0)]$, so $t^\star$ is measurable and feasible. If $\beta^\star>0$,
define $w(x):=g^+(x,\beta^\star)-g^-(x,\beta^\star)\ge 0$ and
$r:=(1-\alpha)-\mathbb E[g^{-}(X,\beta^\star)]$. By the interval condition, $r\in[0,\mathbb E[w(X)]]$. Define the finite measure $\nu$ on $(\mathcal{X}, \mathcal{F})$ by
$$
\nu(B):=\mathbb{E}\left[w(X) \mathbf{1}_{\{X \in B\}}\right]=\int_B w(x) P_X(d x) .
$$
Because $P_X$ is non-atomic and $\nu \ll P_X, \nu$ is also non-atomic. By Sierpiński's theorem \citep{sierpinski1922fonctions}, for any $0 \leq s \leq \nu(\mathcal{X})$, there exists a measurable $A \in \mathcal{F}$ with $\nu(A)=s$. Apply this with $s=r$ to get $A \in \mathcal{F}$ such that $\mathbb{E}\left[w(X) \mathbf{1}_A(X)\right]=r$. Let
\[
t^\star(x):=g^-(x,\beta^\star)+w(x)\mathbf 1_A(x).
\]
Then $t^\star(x)\in\{g^-(x,\beta^\star),g^+(x,\beta^\star)\}\subseteq\Gamma_{\beta^\star}(x)$,
and
\[
\mathbb E[t^\star(X)]
=\mathbb E[g^-(X,\beta^\star)]+\mathbb E[w(X)\mathbf 1_A(X)]
=1-\alpha.
\]
Finally, since in both case, $t^\star(x)\in\Gamma_{\beta^\star}(x)$,
\[
V_X(t^\star(X))-\beta^\star t^\star(X)
=\inf_{u\in[0,1]}\{V_X(u)-\beta^\star u\}\quad\text{a.s.}
\]
Taking expectations gives
\[
\mathbb E[V_X(t^\star(X))]+\beta^\star\bigl((1-\alpha)-\mathbb E[t^\star(X)]\bigr)
=\phi(\beta^\star),
\]
If $\beta^{\star}>0$, we constructed $\mathbb{E}\left[t^{\star}(X)\right]=1-\alpha$, so the Lagrange term vanishes and $\mathbb{E}\left[V_X\left(t^{\star}(X)\right)\right]=\phi\left(\beta^{\star}\right)$. If $\beta^{\star}=0$, then $\mathbb{E}\left[V_X\left(t^{\star}(X)\right)\right]=\phi(0)$ directly.
In both cases, since $\beta^{\star}$ maximizes $\phi$,
$$
\mathbb{E}\left[V_X\left(t^{\star}(X)\right)\right]=\phi\left(\beta^{\star}\right)=\max _{\beta \geq 0} \phi(\beta)=\operatorname{VAL}(\alpha)
$$
by strong duality (Lemma \ref{lem:no-gap}). Hence, $t^\star$ is primal optimal.

Finally, plugging $t^\star$ into \eqref{eq:optimal set} yields the stated optimal prediction sets, with actions $a^\star(x)=a\big(x,t^\star(x)\big)$ as in \eqref{eq:pointwise-optimal action}.
\end{proof}

\subsection{Proofs of Auxiliary Lemmas}
\label{app:auxi_lemma}

\begin{lemma}[Interchange of infimum and expectation for the Lagrangian]
\label{lem:interchange}
Fix $\beta\ge 0$.
Assume $(x,t)\mapsto V_x(t)$ is a normal integrand.
Then
\begin{equation}
\label{eq:interchange}
\inf_{\substack{t:\mathcal X\to[0,1]\\ t\ \text{measurable}}}
\mathbb E\!\left[V_X\big(t(X)\big)-\beta\,t(X)\right]
\;=\;
\mathbb E\!\left[\inf_{u\in[0,1]}\big\{V_X(u)-\beta u\big\}\right].
\end{equation}
Moreover, there exists a measurable selector $t_\beta:\mathcal X\to[0,1]$ such that
\[
t_\beta(x)\in\arg\min_{u\in[0,1]}\big\{V_x(u)-\beta u\big\}
\quad\text{for $P_X$-a.e.\ $x$},
\]
and \eqref{eq:interchange} holds with $t=t_\beta$.
\end{lemma}

{\bf Proof of Lemma~\ref{lem:interchange}.}
We first note that all expectations in \eqref{eq:interchange} are well-defined and finite.
Fix any $a_0\in\mathcal A$. For $t\in(0,1]$ we have $0\le Q_t^x(a_0)\le M(a_0)$, hence
$V_x(t)\le t\,Q_t^x(a_0)+(1-t)M(a_0)\le M(a_0)$.
At $t=0$ we have $V_x(0)=\min_{a\in\mathcal A}M(a)\le M(a_0)$ by definition.
Therefore $0\le V_x(t)\le M(a_0)$ for all $(x,t)\in\mathcal X\times[0,1]$.

Hence $|V_X(t(X))-\beta t(X)|\le M(a_0)+\beta\in L^1(P_X)$. Now, fix $\beta\ge 0$ and define
\[
f(x,u) := V_x(u)-\beta u,\qquad (x,u)\in\mathcal X\times[0,1],
\qquad
\underline f(x) := \inf_{u\in[0,1]} f(x,u).
\]
Since $u\mapsto -\beta u$ is continuous, $f$ is a normal integrand whenever $V$ is. Hence the marginal function
$\underline f$ is $\mathcal F$-measurable (see, e.g., \cite{rockafellar1998variational}, Thm.~14.37). For any measurable $t:\mathcal X\to[0,1]$ we have pointwise
$f(x,t(x))\ge \inf_{u\in[0,1]}f(x,u)=\underline f(x)$, and therefore
\[
\mathbb E\!\left[f\big(X,t(X)\big)\right]\ \ge\ \mathbb E\!\left[\underline f(X)\right].
\]
Taking the infimum over all measurable $t$ yields the ``$\ge$'' direction in \eqref{eq:interchange}.

For the reverse direction, note that for each $x$ the function $u\mapsto f(x,u)$ has closed epigraph, hence is l.s.c.\ on the compact interval $[0,1]$ and attains its minimum. Let
\[
\Gamma(x) := \arg\min_{u\in[0,1]} f(x,u).
\]
Then $\Gamma(x)$ is nonempty and compact. Moreover, because $f$ is a normal integrand, the argmin multifunction
$\Gamma:\mathcal X\rightrightarrows[0,1]$ has a measurable graph (again see \cite{rockafellar1998variational}, Thm.~14.37).
By the Kuratowski--Ryll--Nardzewski measurable selection theorem \citep{kuratowski1965general}, there exists a measurable selector
$t_\beta:\mathcal X\to[0,1]$ such that $t_\beta(x)\in\Gamma(x)$ for $P_X$-a.e.\ $x$. Consequently,
$f\big(x,t_\beta(x)\big)=\underline f(x)$ a.e., and integrating yields equality in \eqref{eq:interchange} (and attainment).
\qed

\begin{lemma}[No duality gap for the average-coverage problem under non-atomic $P_X$]
\label{lem:no-gap}
Assume $P_X$ is non-atomic and that  $(x,t)\mapsto V_x(t)$ is a normal integrand.
Fixing $\alpha\in(0,1)$, we have
\[
\operatorname{VAL}(\alpha) = \sup_{\beta\ge 0}\ \phi(\beta),
\]
i.e., there is no duality gap between the primal problem and its Lagrange dual.
\end{lemma}

{\bf Proof of Lemma~\ref{lem:no-gap}.}
We introduce a convex relaxation that allows randomization of $t$ conditional on $X$.
Let $\mathcal P([0,1])$ be the set of Borel probability measures on $[0,1]$, and let $\mathcal M$ be the set of (universally) measurable stochastic kernels
$x\mapsto \mu_x\in\mathcal P([0,1])$.
Consider the relaxed value
\[
\mathrm{VAL}_{\mathrm{rel}}(\alpha)
:=\inf_{\mu_{(\cdot)}\in\mathcal M}
\Big\{
\mathbb E\!\big[\textstyle\int V_X(u)\,\mu_X(\mathrm du)\big]
:\ 
\mathbb E\!\big[\textstyle\int u\,\mu_X(\mathrm du)\big]\ge 1-\alpha
\Big\}.
\tag{$\mathsf{P}_{\mathrm{rel}}$}
\]
The objective and constraint are linear in $\mu_{(\cdot)}$, so $(\mathsf{P}_{\mathrm{rel}})$ is a convex program.
Moreover, $\mu_x=\delta_1$ is strictly feasible when $\alpha>0$, hence Slater's condition holds.
By Fenchel--Rockafellar duality for convex integral functionals with a single linear moment constraint in the framework of normal integrands
(see, e.g., \cite{rockafellar1998variational} Thm.~11.39.),
there is no duality gap for $(\mathsf{P}_{\mathrm{rel}})$ and
\[
\mathrm{VAL}_{\mathrm{rel}}(\alpha)
=
\sup_{\beta\ge 0}\ 
\inf_{\mu_{(\cdot)}\in\mathcal M}
\left\{
\mathbb E\!\Big[\textstyle\int \big(V_X(u)-\beta u\big)\,\mu_X(\mathrm du)\Big]+\beta(1-\alpha)
\right\}.
\]

Fix $\beta\ge 0$ and write $f_\beta(x,u):=V_x(u)-\beta u$.
For any kernel $\mu_{(\cdot)}$ we have pointwise
$\int f_\beta(x,u)\,\mu_x(\mathrm du)\ge \inf_{u\in[0,1]} f_\beta(x,u)$, hence
\[
\inf_{\mu_{(\cdot)}\in\mathcal M}\ \mathbb E\!\Big[\textstyle\int f_\beta(X,u)\,\mu_X(\mathrm du)\Big]
\ \ge\
\mathbb E\!\Big[\inf_{u\in[0,1]} f_\beta(X,u)\Big].
\]
Conversely, by Lemma~\ref{lem:interchange} there exists a measurable selector $t_\beta$ with
$t_\beta(x)\in\arg\min_{u\in[0,1]} f_\beta(x,u)$ for $P_X$-a.e.\ $x$; taking $\mu_x=\delta_{t_\beta(x)}$ yields equality.
Therefore,
\[
\inf_{\mu_{(\cdot)}\in\mathcal M}
\left\{
\mathbb E\!\Big[\textstyle\int \big(V_X(u)-\beta u\big)\,\mu_X(\mathrm du)\Big]+\beta(1-\alpha)
\right\}
=
\beta(1-\alpha)+\mathbb E\!\left[\inf_{u\in[0,1]}\big\{V_X(u)-\beta u\big\}\right]
=\phi(\beta),
\]
and hence $\mathrm{VAL}_{\mathrm{rel}}(\alpha)=\sup_{\beta\ge 0}\phi(\beta)$.

It remains to show that relaxation does not change the value.
Any measurable $t$ yields a feasible kernel $\mu_x=\delta_{t(x)}$, so $\mathrm{VAL}_{\mathrm{rel}}(\alpha)\le \mathrm{VAL}(\alpha)$.
For the reverse inequality, fix any feasible kernel $\mu_{(\cdot)}$ in $(\mathsf{P}_{\mathrm{rel}})$ and consider the two bounded measurable functions
$h_1(x,u)=1-u$ and $h_2(x,u)=V_x(u)$ on $\mathcal X\times[0,1]$ (boundedness holds in our setting, since $V_x(t)\le M(a_0)$ for any fixed $a_0\in\mathcal A$), hence also are nonnegative normal integrands.
Since $P_X$ is non-atomic, the Dvoretzky--Wald--Wolfowitz purification theorem \citep{dvoretzky1951elimination, balder1985elimination} implies that there exists a measurable
$t:\mathcal X\to[0,1]$ such that 
\[
\mathbb E\!\left[V_X\big(t(X)\big)\right]
\leq \mathbb E\!\Big[\textstyle\int V_X(u)\,\mu_X(\mathrm du)\Big],
\qquad
\mathbb E\!\big[1 - t(X)\big]
\leq \mathbb E\!\Big[\textstyle\int (1-u)\,\mu_X(\mathrm du)\Big].
\]
The second inequality rewrites as
\[
\mathbb E\big[t(X)\big] \;\ge\; \mathbb E\!\Big[\textstyle\int u\,\mu_X(\mathrm du)\Big] \;\ge\; 1-\alpha,
\]
so $t$ is feasible for \eqref{eq:pop-problem}, while the first inequality shows that its objective value is no larger than that of $\mu(\cdot)$. Hence $\mathrm{VAL}(\alpha)\le \mathrm{VAL}_{\mathrm{rel}}(\alpha)$. Combining the two inequalities yields $\mathrm{VAL}(\alpha)=\mathrm{VAL}_{\mathrm{rel}}(\alpha)=\sup_{\beta\ge 0}\phi(\beta)$.
\qed

\section{Autonomous driving: construction details}
\label{app:bdd-toy}
This appendix provides the construction of the 3-bit hazard label $Y=(Y_a,Y_\ell,Y_r)\in\{0,1\}^3$ and the black-box probability model $f_x(y)$ used in the autonomous-driving experiment.

{\bf ROI parametrization.}
For any bounding box $b=(x^{(1)},y^{(1)},x^{(2)},y^{(2)})$ in an image of width $W$ and height $H$, we define
\[
u=\frac{x^{(1)}+x^{(2)}}{2W},\qquad
v_{\mathrm{bot}}=\frac{y^{(2)}}{H},\qquad
\rho=\frac{y^{(2)}-y^{(1)}}{H},
\]
where $u$ is the normalized horizontal center, $v_{\mathrm{bot}}$ is the normalized bottom coordinate, and $\rho$ is the normalized box height. Intuitively, $v_{\mathrm{bot}}$ and $\rho$ serve as simple proxies for proximity: objects that are larger and closer to the bottom of the image are treated as closer to the ego vehicle.

We define the following regions:
\[
\text{Left: }u\le \tfrac13,\qquad \text{Right: }u\ge \tfrac23,
\]
\[
\text{Ahead-close: }\tfrac13\le u\le \tfrac23,\ \ v_{\mathrm{bot}}\ge 0.6,\ \ \rho\ge 0.15,
\]
\[
\text{Side-close: }v_{\mathrm{bot}}\ge 0.5,\ \ \rho\ge 0.10.
\]

{\bf Ground-truth hazard bits from BDD100K annotations.}
For each image, we extract the annotated objects (category and 2D bounding box) from the BDD100K JSON labels.
We map categories \texttt{person} and \texttt{rider} to \emph{pedestrian} objects, and \texttt{car}, \texttt{truck}, \texttt{bus}, \texttt{train}, \texttt{motor}, \texttt{bike} to \emph{vehicle} objects.
The hazard bits are defined by the existence of at least one qualifying object in the corresponding ROI:
\[
Y_a=1\ \Longleftrightarrow\ \exists\ \text{(pedestrian or vehicle) box in the Ahead-close region},
\]
\[
Y_\ell=1 \Longleftrightarrow \exists \text{vehicle box with }(\text{Left})\ \&\ (\text{Side-close}),\ \
Y_r=1 \Longleftrightarrow \exists \text{vehicle box with }(\text{Right})\ \&\ (\text{Side-close}).
\]

{\bf Black-box probability model $f_x$.}
To obtain a black-box probability model $f_x$, 
we  run a pretrained \texttt{YOLO11} detector \citep{yolo11_ultralytics}  on each image and extract three scalar scores:
$s_a$ is the maximum confidence among detected pedestrian/vehicle boxes that satisfy the ahead-close rule; $s_\ell$ and $s_r$ are the analogous maxima over detected vehicle boxes in the left/right side-close regions. To convert detector scores into calibrated probabilities, we split the data into training/calibration/test, and we fit an isotonic regression map on training subset, $\phi_k(s)\approx \mathbb P(Y_k=1\mid s_k=s)$ for each bit $k\in\{a,\ell,r\}$. On the calibration and test subset, we set $p_k(x)=\phi_k(s_k(x))$ and define an 8-class distribution by conditional independence,
\[
f_x(y)= \prod_{k\in\{a,\ell,r\}} p_k(x)^{y_k}\big(1-p_k(x)\big)^{1-y_k},\
y\in\{0,1\}^3.\]

{\bf Actions and loss.}
We consider actions $\mathcal A=\{\textsf{STOP},\textsf{LEFT},\textsf{RIGHT},\textsf{KEEP}\}$.
Let $M\gg 1$ denote a catastrophic collision cost and define, for $y=(y_a,y_\ell,y_r)\in\{0,1\}^3$,
$\ell(\textsf{KEEP},y) := M\,\mathbf 1\{y_a=1\}$,
$$\ell(\textsf{LEFT},y) := M\,\mathbf 1\{y_\ell=1\}+c_{\text{turn}}+c_{\text{unnec}}\mathbf 1\{y_a=0\},$$
$$\ell(\textsf{RIGHT},y) := M\,\mathbf 1\{y_r=1\}+c_{\text{turn}}+c_{\text{unnec}}\mathbf 1\{y_a=0\},$$
$$\ell(\textsf{STOP},y) := c_{\text{stop,free}}\mathbf 1\{y_a=0\}+c_{\text{stop,block}}\mathbf 1\{y_a=1\}.$$
We use $M=60$, $c_{\text{turn}}=3$, $c_{\text{unnec}}=2$, $c_{\text{stop,free}}=6$, and $c_{\text{stop,block}}=2$.

\end{document}